\documentclass[sigconf]{acmart}
\AtBeginDocument{%
  }

\copyrightyear{2026}
\acmYear{2026}
\setcopyright{cc}
\setcctype{by}
\acmConference[MM '26]{Proceedings of the 34th ACM International Conference on Multimedia}{November 10--14, 2026}{Rio de Janeiro, Brazil}
\acmBooktitle{Proceedings of the 34th ACM International Conference on Multimedia (MM '26), November 10--14, 2026, Rio de Janeiro, Brazil}
\acmDOI{10.1145/3767308.3835714}
\acmISBN{979-8-4007-2213-4/2026/11}

\usepackage{booktabs}
\usepackage{multirow}
\usepackage{makecell}
\usepackage{algorithm}
\usepackage{algorithmic}
\usepackage{xspace}
\usepackage{amsmath}

\newcommand{\method}{$\operatorname{SynMDiff}$\xspace}

\newtheorem{theorem}{Theorem}

\newtheorem{proposition}{Proposition}
\newtheorem{corollary}{Corollary}
\newtheorem{definition}{Definition}

\usepackage{colortbl}
\definecolor{best}{RGB}{255, 230, 230}
\definecolor{second}{RGB}{230, 240, 255}

\newtheorem*{restatedpropone}{Proposition~\ref{prop:asynchronous}}
\newtheorem*{restatedproptwo}{Proposition~\ref{prop:synchronous}}
\newtheorem*{restatedcor}{Corollary~\ref{cor:restricted}}
\newtheorem*{restatedthm}{Theorem~\ref{theo:consistency}}

\begin{document}

\title{Synchronous Multi-view Neural Diffusion}

\author{Yongquan Shi}
\orcid{0009-0005-1032-3887}

\affiliation{%
	\institution{Fujian Normal University}%
	\city{Fuzhou}%
	\country{China}
}

\affiliation[obeypunctuation=true]{%
	\institution{Fuzhou University},~%
	\city{Fuzhou},~%
	\country{China}
}
\email{losparksayoji@outlook.com}

\author{Weijun Huang}
\affiliation{%
	\institution{Fuzhou University}
	\city{Fuzhou}
	\country{China}}
\email{weijunhuang22@gmail.com}

\author{Yueyang Pi}
\affiliation{%
	\institution{Fuzhou University}
	\city{Fuzhou}
	\country{China}}
\email{piyueyangcc@163.com}

\author{Wendi Zhao}
\affiliation{%
	\institution{Fuzhou University}
	\city{Fuzhou}
	\country{China}}
\email{241010030@fzu.edu.cn}

\author{Yiqing Shi}
\correspondingauthor
\affiliation{%
	\institution{Fujian Normal University}
	\city{Fuzhou}
	\country{China}}
\email{417shelly@gmail.com}

\author{Shiping Wang}
\affiliation{%
	\institution{Fuzhou University}
	\city{Fuzhou}
	\country{China}}
\email{shipingwangphd@163.com}

\renewcommand{\shortauthors}{Yongquan Shi, Weijun Huang, Yueyang Pi, Wendi Zhao, Yiqing Shi and Shiping Wang.}

\begin{abstract}
  Multi-view learning seeks to learn more comprehensive representations by exploiting the complementarity and consistency across diverse modalities or views.
  However, existing multi-view fusion strategies treat intra- and inter-view fusion as independent stages, without simultaneously considering the evolution within views and the dependency across views.
  Such an asynchronous fusion paradigm inevitably constrains cross-view interactions due to conflicting view-specific structural inductive biases. As a result, information flow is prone to distortion and compression along intermediate pathways, confining the model to learn within a restricted solution space.
  To address this, we propose Synchronous Multi-view Neural Diffusion (SynMDiff), which conceptualizes the multi-view feature space as a unified dynamical system driven by a diffusion process.
  By modeling the diffusion flow across arbitrary dyadic feature interactions in a joint space, SynMDiff enables the concurrent and adaptive intra- and inter-view information fusion.
  While a direct implementation of this synchronized mechanism incurs prohibitive computational costs, we further introduce an energy-based topological sampling strategy and an Ego-Net style centralized training architecture, ensuring both efficiency and scalability during learning and inference.
  Due to its conceptual elegance and computational efficacy, evaluations on real-world datasets demonstrate that SynMDiff outperforms the baselines by a large margin\footnote{Code is available at \url{https://github.com/LosparkSayoji/SynMDiff}.}.
\end{abstract}



\maketitle
\hypersetup{pdfauthor={Yongquan Shi, Weijun Huang, Yueyang Pi, Wendi Zhao, Yiqing Shi, Shiping Wang}}

\section{Introduction}
In real-world scenarios, entities are represented in diverse modalities and can be characterized by features from various views.
Typical examples include multi-angle camera arrays in autonomous driving \cite{huang2021bevdet}, text-video pairs in multimedia retrieval \cite{jin2024mv}, and multi-omics profiles in AI-driven drug discovery \cite{wang2021mogonet}.
As a pivotal paradigm, multi-view learning aims to integrate information from diverse views, capturing their inherent consistency and complementarity to yield a more comprehensive representation than any single view can provide \cite{blum1998combining,xu2013survey,gao2015multi,ma2026multi,Zhao2024DynamicGG}.
In practice, the relationships between views are often non-linear and higher-order, posing formidable challenges in modeling these intricate inter-view correlations across heterogeneous feature spaces.

Surprisingly, an in-depth and principled exploration of multi-view fusion strategies has been largely overlooked.
Existing multi-view fusion methods \cite{lu2024towards,shi2025information,wang2025MEGNN} typically rely on independent encoders to extract features from different views, followed by \textit{post-hoc} intra-instance fusion of the resulting representations via learned scalar weights (\textit{e.g.}, attention scores).
In this way, intra- and inter-view fusion are separated into two independent stages, forming a view-decoupled asynchronous fusion paradigm.
We provide an intuitive illustration of this architecture in Fig.~\ref{intro}(a).

\begin{figure}[!t]
	\centering
	\includegraphics[width=0.946\linewidth]{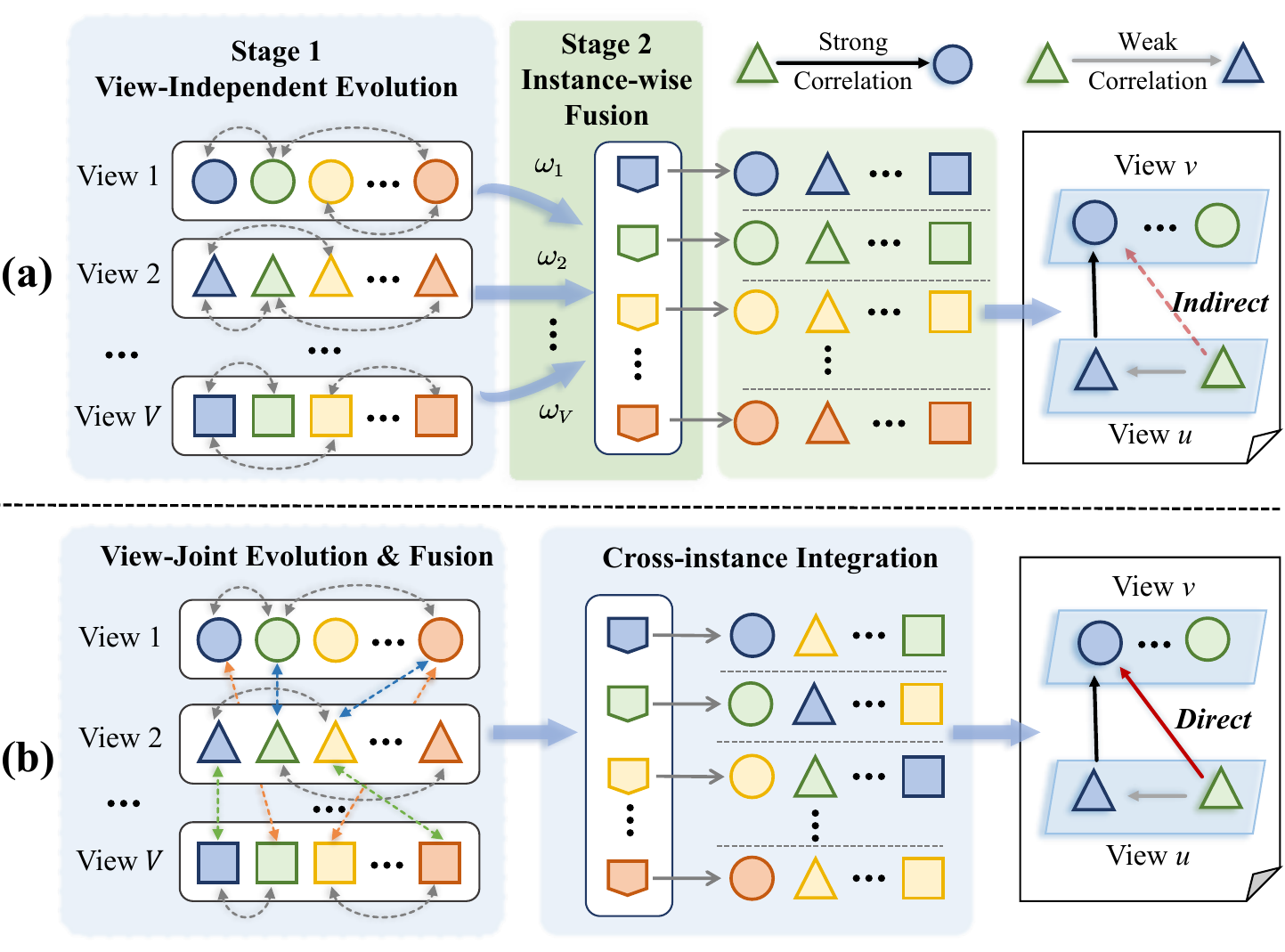}
	\Description{
		A two-panel schematic comparing asynchronous and synchronous
		multi-view fusion. Panel (a) shows view-specific information being
		processed separately and then fused in two stages. Indirect
		cross-view pathways contain conflicts caused by incompatible
		structural inductive biases, resulting in distorted information
		flow. Panel (b) shows the proposed synchronous paradigm, which
		jointly models intra-view and inter-view interactions and provides
		direct information-flow pathways among arbitrary views.
	}
	\caption{
			Comparison of two multi-view fusion paradigms.
			(a) The two-stage asynchronous paradigm does not explicitly model
			cross-view interactions. Conflicts among view-specific structural
			inductive biases can therefore distort information flow along intermediate
			pathways.
			(b) The proposed synchronous paradigm jointly models intra-view and
			inter-view fusion, enabling direct interactions between arbitrary views.
		}
	\label{intro}
\end{figure}

However, despite its conceptual simplicity and ease of implementation, this asynchronous fusion strategy substantially limits the expressive power of the model, as inter-view interactions are inevitably influenced by the structural inductive biases imposed within each view.
More specifically, within such an asynchronous fusion framework, interactions between arbitrary cross-view feature pairs are not explicitly modeled but are instead decomposed into a two-stage process involving intra-view and inter-view operations.
As a result, cross-view information must propagate through indirect pathways, where it is prone to distortion or compression due to \textit{bottleneck effects} \cite{tishby2000information}.
This constitutes an inherently constrained information flow, reflecting structural limitations imposed along the interaction paths.
Consider a simple yet critical scenario: even if sample $i$ in view $v$ is strongly correlated with sample $j$ in view $u$, if samples $i$ and $j$ within view $u$ are weakly correlated or even unrelated, then sample $i$ in view $v$ can hardly access information from sample $j$ through sample $i$ in view $u$.
More critically, such structural conflicts are prevalent in real-world multi-view settings involving heterogeneous data.

To address these limitations, we propose \textbf{Syn}chronous \textbf{M}ulti-view Neural \textbf{Diff}usion (\method), a novel framework that reformulates joint multi-view feature evolution through the lens of continuous-time ordinary differential equations.
Inspired by diffusion dynamics, \method explicitly models arbitrary pairwise feature interactions as adaptive diffusion flows in a unified multi-view space, thereby enabling concurrent intra- and inter-view fusion while mitigating the influence of view-specific inductive biases.
This unified formulation allows information to propagate directly across views without relying on view-specific intermediate representations or predefined interaction pathways.
Despite its theoretical appeal, a direct implementation of synchronous diffusion incurs prohibitive computational overhead.
To address this, we generalize the diffusion process to model information propagation over a complete graph and introduce an energy-based topological sampling mechanism to prune redundant interactions and optimize diffusion pathways.
Furthermore, we design an \textit{Ego-Net}-style centralized architecture \cite{hamilton2017inductive}, enabling \method to achieve efficient training and inference in large-scale scenarios.
Our main contributions are summarized below:
\begin{itemize}
	\item We propose synchronous multi-view neural diffusion to mitigate the influence of view-specific inductive biases when modeling cross-view interactions.

	\item To alleviate the computational overhead, we develop a topological sampling strategy with a centralized training architecture for efficient deployment.

	\item Extensive experiments on real-world datasets across diverse multi-view scenarios demonstrate the effectiveness of the proposed method.
\end{itemize}

\section{Related Work}

\subsection{Multi-view Learning}
There are numerous studies dedicated to facilitating the feature fusion within multi-view settings.
Early efforts in this field used linear weighting mechanisms \cite{tao2017scalable} or learned adaptive scalar weights \cite{nie2017multi,huang2021embedding} to balance the importance of different views. To bridge the gap between disparate feature spaces, several approaches projected multi-view data into a unified low-dimensional subspace via linear transformations \cite{yang2019adaptive} or captured intricate patterns through non-linear subspace learning \cite{xie2018hyper}.
Beyond shallow mappings, Xu et al.~\cite{xu2020deep} incorporate a set of shared parameters alongside view-specific independent encoders to extract cross-view consistencies.
Some research also explored multi-stage training paradigms \cite{xu2023progressive} to refine the coordination between disparate views.
Recently, a burgeoning trend has focused on integrating graph-based models into multi-view frameworks, demonstrating their superior capacity to capture consistency and complementarity via topological structures.
Following the graph-based co-training paradigm of Li et al.~\cite{li2020co}, subsequent work has focused on constructing unified topologies for multi-view data~\cite{chen2023learnable,lu2024generative,chen2023joint,Zhao2025IncompleteAU}, with extensions toward interpretability~\cite{wu2023interpretable} and deeper architectures~\cite{shi2025information}.
These methods either model different views independently or learn a shared subspace, neglecting the dynamic cross-view interactions.

Recently, the latest works \cite{lu2024towards,Lai_Li_Wang_Wang_2026} have attempted to incorporate global attention mechanisms to capture fine-grained interactions within multi-view data. However, these methods still follow a two-stage asynchronous fusion paradigm, where the modeled cross-view interactions are inevitably influenced by view-specific structural inductive biases.

\subsection{Neural Diffusion Models}
Several recent efforts explore diffusion-based learning, where continuous dynamics serve as an inductive bias for representation learning \cite{long2018pde,liu2016learning}. Most of these works formulate models based on continuous diffusion processes governed by Partial Differential Equations (PDEs) \cite{chamberlain2021grand,eliasof2021pde,thorpe2022grand++} or develop novel frameworks from a diffusive perspective \cite{gasteiger2019diffusion,atwood2016diffusion}. In these settings, diffusion is typically interpreted as feature-level interaction or characterized as information propagation over graphs via geometric properties \cite{yang2022geometric}.

Our work follows the trajectory of PDE-inspired modeling, where the key originality lies in:
(i) we introduce neural diffusion flows to efficiently model intricate cross-view interactions among heterogeneous features in multi-view data;
(ii) we propose a synchronous diffusion mechanism that enables the concurrent fusion of intra- and inter-view features.
To the best of our knowledge, this is the first attempt to formulate the multi-view feature space as a unified dynamical system, allowing intra- and inter-view interactions to be captured in a fully synchronized manner.

\section{Preliminary}
We start by studying the diffusion process in a system $\Omega$, which is typically assumed to be a Riemannian manifold \cite{rosenberg1997laplacian}. Informally, the diffusion process describes the movement of mass from regions of higher to lower
concentration.
Consider $x(t):\Omega \to [0,\infty)$ as a family of scalar-valued functions representing temperature distributions, and let $x(p,t)$ denote the temperature at point $p \in \Omega$ at time $t$. The heat diffusion process is described by the following partial differential equation with boundary conditions \cite{chamberlain2021grand}:
\begin{equation}\label{heat diffusion}
	\begin{aligned}
		\frac{\partial x(p, t)}{\partial t} = \text{div}[S(p, x(p, t), t)  \nabla x(p, t)],\\
		\text{s. t. } x(p,0) = x_0(p), \quad t \ge 0.
	\end{aligned}
\end{equation}
Here, $x_0(\cdot)$ denotes the initial temperature distribution, and $S(\cdot)$ is the concentration-dependent diffusion coefficient.
The gradient operator $\nabla$ measures the local temperature variations that drives the flow, while the divergence operator $\text{div}(\cdot)$ aggregates the heat flux at each point.
To account for more practical scenarios, we consider a discrete data space $\mathbf{X}$ with $N$ instances, Eq.~(\ref{heat diffusion}) can be further written as
\begin{equation}\label{feature diffusion}
	\frac{\partial \mathbf{x}_i(t)}{\partial t} = \sum_{j=1}^{N} \mathbf{S}_{ij}(\mathbf{X}(t)) (\mathbf{x}_j(t) - \mathbf{x}_i(t)),
\end{equation}
where feature discrepancies between instances are interpreted as temperature variations.
Such a discrete diffusion process can serve as a
\textit{relational inductive bias} \cite{coifman2006diffusion}, encouraging the model to exploit information across instances.
To solve Eq.~(\ref{feature diffusion}), a commonly approach is to employ the explicit forward Euler scheme, which yields
\begin{equation}\label{diffusion forward}
	\mathbf{x}_i^{(t+1)} = \left(1 - \tau \sum_{j=1}^{N} \mathbf{S}_{ij}^{(t)}\right) \mathbf{x}_i^{(t)} + \tau \sum_{j=1}^{N} \mathbf{S}_{ij}^{(t)} \mathbf{x}_j^{(t)}.
\end{equation}
According to \cite{wudifformer}, for any $\tau \in (0,1)$, the above numerical iteration is guaranteed to converge.
In this context, Eq.~(\ref{diffusion forward}) provides a basic formulation for the diffusion process in discrete feature spaces.

\section{Method}

\subsection{Rethinking the Decoupled Paradigm in Multi-view Fusion}
For a multi-view feature space $\mathcal{X} = \{\mathbf{X}^{(v)} \in \mathbb{R}^{N \times D^{(v)}} \}_{v=1}^V$ comprising $V$ views and $N$ instances, prevalent fusion strategies first process each view with independent encoders, followed by instance-wise aggregation (\textit{e.g.}, learned scalar weights or attention scores), forming a widely adopted paradigm in recent studies \cite{lu2024towards,wang2025MEGNN,shi2025information}.
For analytical clarity, we consider a single interaction step here.
In this context, a unified representation $\mathbf{Z} \in \mathbb{R}^{N \times D}$ is obtained by fusing multi-view features, formulated as
\begin{equation}
	\mathbf{Z} = \mathcal{F} \left( f_{\theta_1}(\mathbf{X}^{(1)}), f_{\theta_2}(\mathbf{X}^{(2)}), \dots, f_{\theta_V}(\mathbf{X}^{(V)}) \right),
\end{equation}
where $\mathcal{F}(\cdot)$ denotes the fusion operator and $f_{\theta_v}(\cdot)$ is a view-specific transformation parameterized by $\theta_v$.
Since distinct views are isolated during feature extraction, with intra- and inter-view information flow being sequentially staged, we characterize this paradigm as View-Decoupled Fusion.
\begin{definition}\label{definition:view_decoupled_fusion}
	(View-Decoupled Fusion)
	A multi-view fusion strategy constitutes View-Decoupled Fusion adhering to the following operator factorization:
	\begin{itemize}
		\item View-Independent Evolution.
		The latent transformation $f_{\theta_v}$ is independent of the state of any external view $u$:
		\begin{equation}
			\frac{\partial f_{\theta_v}(\mathbf{X}^{(v)}_i)}{\partial \mathbf{X}^{(u)}_j} = \mathbf{0}, \quad \forall u \neq v, \forall (i, j).
		\end{equation}

		\item Intra-instance View Integration.
		Cross-view synergy is confined to point-wise mapping over co-indexed fibers:
		\begin{equation}
			\mathbf{z}_i = \mathcal{F} \left( f_{\theta_1}(\mathbf{x}^{(1)}_i), \dots, f_{\theta_V}(\mathbf{x}^{(V)}_i) \right).
		\end{equation}
	\end{itemize}
\end{definition}

\noindent Although this paradigm offers architectural intuitiveness and ease of implementation, it is intrinsically a \textit{post-hoc} hierarchical strategy.
Formally, cross-view interactions are not explicitly modeled; instead, they are implicitly factorized into a two-stage procedure involving intra-view propagation followed by cross-view alignment:
\begin{equation}
	g(x_i^{(v)}, x_j^{(u)}) \approx \phi
	\left[\psi_1(x_i^{(u)}, x_j^{(u)}), \psi_2(x_i^{(u)}, x_i^{(v)})\right].
\end{equation}
Here, $g(\cdot,\cdot)$ denotes the ideal cross-view interaction; $\psi_1(\cdot,\cdot)$ captures intra-view propagation, $\psi_2(\cdot,\cdot)$ models cross-view alignment between co-indexed samples, and $\phi(\cdot,\cdot)$ aggregates these components to produce the final approximation.
Evidently, such an approximation significantly limits the expressive power of the model, causing the learned cross-view interactions to be highly susceptible to conflicting view-specific structural inductive biases, and leading to distortion or compression of information flow along the intermediate pathways.
Transcending this decoupling paradigm, our core idea is to treat all views as a unified whole rather than as independent components, thereby directly and explicitly modeling $g(\cdot,\cdot)$ and formulating a Global Coupled Fusion.
\begin{definition}\label{definition:global_coupled_fusion}
	(Global Coupled Fusion)
	For any instance $i$, a multi-view fusion strategy constitutes Global Coupled Fusion if it is characterized by the following properties:
	\begin{itemize}
		\item View-Joint Evolution.
		The latent transformations $f_{\theta_v}$ exhibit inter-view dependencies:
		\begin{equation}
			\frac{\partial f_{\theta_v}(\mathbf{X}^{(v)}_i)  }{\partial \mathbf{X}^{(u)}_j} \neq \mathbf{0}, \quad \forall (u, v), \forall (i, j).
		\end{equation}

		\item Cross-instance View Interaction.
		View integration is able to capture inter-sample dependencies:
		\begin{equation}
			\mathbf{z}_i = \mathcal{F} \left( \bigcup_{v=1}^V \{ f_{\theta_v}(\mathbf{x}^{(v)}_1), \dots, f_{\theta_v}(\mathbf{x}^{(v)}_N) \} \right).
		\end{equation}
	\end{itemize}
\end{definition}
\noindent Following Definition~\ref{definition:global_coupled_fusion}, we next present a diffusion-based formulation for modeling arbitrary pairwise feature interactions in the multi-view space.

\subsection{Synchronous Multi-view Neural Diffusion}
Recalling Eq.~\eqref{feature diffusion}, feature interactions can be effectively modeled through the lens of diffusion dynamics \cite{wudifformer}.
In previous studies \cite{lu2024towards}, the multi-view diffusion process is interpreted as
\begin{equation}\label{lu diffsuion}
	\begin{aligned}
		\frac{\partial \mathbf{x}^{(v)}_i(t)}{\partial t} &= \sum_{j=1}^{N} \mathbf{S}^{(v)}_{ij}(t) (\mathbf{x}^{(v)}_j(t) - \mathbf{x}^{(v)}_i(t))\\
		&+ \sum_{u=1}^{V} \mathbf{P}_{vu}(t) (\mathbf{x}^{(u)}_i(t) - \mathbf{x}^{(v)}_i(t)),
	\end{aligned}
\end{equation}
where $\mathbf{P}_{vu}(t) \in \mathbb{R}$ denotes a scalar fusion coefficient between the $v$-th and the $u$-th views for instance $i$ at time $t$.
Evidently, this formulation constitutes a View-Decoupled Fusion paradigm, which we term \textit{asynchronous} multi-view diffusion.
\begin{proposition}\label{prop:asynchronous}
	The asynchronous multi-view diffusion described in Eq.~\eqref{lu diffsuion} constitutes a View-Decoupled Fusion paradigm.
\end{proposition}
\noindent The proof is given in \textbf{Appendix~\ref{app:proof-prop1}}.
In contrast, we aim to construct a diffusion mechanism that operates across views and samples simultaneously, as formalized in Definition~\ref{definition:global_coupled_fusion}.
Specifically, at each time step, the feature of instance $i$ in the $v$-th view not only participates within the intra-view diffusion, but also contributes as part of the entire multi-view space:
\begin{equation}\label{objective}
	\frac{\partial \mathbf{x}^{(v)}_i(t)}{\partial t} = \sum_{u=1}^{V} \alpha^{(u)}(t) \cdot \mathcal{D}^{(u)}_i(\mathcal{X}(t)),
\end{equation}
where $\alpha^{(u)}$ is the trade-off coefficient across views, $\mathcal{D}^{(u)}_i(\cdot)$ denotes the diffusion operator of instance $i$, which acts on the global state of the multi-view feature space $\mathcal{X}$ and quantifies the information flow contributed by all instances from view $u$.
By further expanding $\mathcal{D}^{(u)}_i(\cdot)$, we can obtain a diffusion formulation analogous to Eq.~(\ref{feature diffusion}):
\begin{equation}\label{global diffusion}
	\!\! \frac{\partial \mathbf{x}^{(v)}_i(t)}{\partial t} \!=\! \sum_{u=1}^{V} \alpha^{(u)}(t) \sum_{j=1}^{N} \mathbf{S}^{(v, u)}_{i, j}(t) (\mathbf{x}^{(u)}_j(t) \!-\! \mathbf{x}^{(v)}_i(t)).
\end{equation}
We refer to this formulation as \textit{synchronous} multi-view diffusion. Here, $\mathbf{S}^{(v, u)}_{i, j}(t)$ represents a dual-coupling weight spanning across both views and instances, which quantifies the influence of the $j$-th instance in view $u$ on the $i$-th instance in view $v$. Such mechanism facilitates concurrent intra- and inter-view information propagation, effectively establishing a global receptive field across the entire multi-view joint space.
Naturally, this synchronous multi-view diffusion mechanism represents a Global Coupled Fusion paradigm.
\begin{proposition}\label{prop:synchronous}
	The synchronous multi-view diffusion characterized by Eq.~\eqref{global diffusion} constitutes a Global Coupled Fusion.
\end{proposition}
\noindent The proof is given in \textbf{Appendix~\ref{app:proof-prop2}}.
In this context, we can further characterize the relationship between the synchronous and asynchronous diffusion mechanisms.
\begin{corollary}\label{cor:restricted}
	The asynchronous multi-view diffusion in Eq.~\eqref{lu diffsuion} can be regarded as a restricted case of the synchronous multi-view diffusion Eq.~\eqref{global diffusion}.
\end{corollary}
\noindent
See \textbf{Appendix~\ref{app:proof-cor1}} for the proof.
Similar to Eq.~(\ref{diffusion forward}), we can employ the forward Euler scheme to numerically solve Eq.~(\ref{global diffusion}):
\begin{equation}\label{forward}
	\begin{aligned}
		\! \mathbf{x}^{(v)}_i(t\!+\!1) \!&=\! \left[ 1 \!-\! \tau \sum_{u=1}^{V} \alpha^{(u)}(t) \sum_{j=1}^{N} \mathbf{S}^{(v, u)}_{i, j}(t) \right] \! \mathbf{x}^{(v)}_i(t)\\
		&+ \tau \sum_{u=1}^{V} \alpha^{(u)}(t) \sum_{j=1}^{N} \mathbf{S}^{(v, u)}_{i, j}(t) \mathbf{x}^{(u)}_j(t).
	\end{aligned}
\end{equation}
Typically, the diffusion coefficient here can be efficiently parameterized via an attention mechanism \cite{chamberlain2021grand}:
\begin{equation}
	\mathbf{S}^{(v, u)}_{i, j} (t) = \operatorname{softmax}\left(\frac{(\mathbf{W}_K\mathbf{x}_i^{(v)}(t))^{\mathrm{T}}\mathbf{W}_Q\mathbf{x}_j^{(u)}(t)}{d_k}\right),
\end{equation}
where $\mathbf{W}_K$ and $\mathbf{W}_Q$ are learned matrices, and $d_k$ is a hyperparameter determining the dimension of $\mathbf{W}_K$.
However, a direct implementation of Eq.~\eqref{forward} incurs $O(V^2N^2)$ computational and memory complexity, as it requires storing a dense coefficient matrix and performing exhaustive interactions over the multi-view space.
This quadratic scaling poses a significant scalability bottleneck for large-scale applications, a challenge also observed in \cite{lu2024towards}\footnote{A detailed analysis of the computational cost is provided in \textbf{Appendix~\ref{app:complexity}}.}.
In the following subsections, we will detail the strategies to address these computational hurdles.

\subsection{Energy-based Topological Sampling}
In practice, the diffusion process in discrete feature space can be generalized to information propagation over a complete graph $\mathcal{G}$ \cite{gasteiger2019diffusion}, where the diffusivity coefficient $\mathbf{S}$ serves as a fully connected adjacency matrix.
In this way, the forward process induced by Eq.~(\ref{forward}) admits a graph neural network (GNN)-style formulation\footnote{Applying forward Euler discretization to this ODE yields $\mathbf{X}(k+1)=(\mathbf{I}-\Delta t\,\mathbf{L})\mathbf{X}(k)$, which corresponds to a GNN-style neighborhood aggregation step.}:
\begin{equation}
	\mathbf{x}^{(v)}_i(t + 1) = \operatorname{UPD}  \left(
	\operatorname{AGG}_{u=1,j=1}^{V,N}  \left( \mathbf{x}^{(u)}_{j}(t); \mathbf{S}_{i,j}^{v,u}(t)  \right) \right) ,
\end{equation}
where $\operatorname{UPD}(\cdot)$ denotes an update function typically comprising a linear transformation followed by a non-linear activation, while $\operatorname{AGG}(\cdot)$ represents a global aggregation operator.
In this light, the diffusion process in Eq.~\eqref{global diffusion} can be further interpreted as a gradient flow that minimizes the Dirichlet energy over the view-joint complete graph $\mathcal{G}$ \cite{fu2023implicit}\footnote{Specifically, for $\mathcal{E}(\mathbf{X})=\frac{1}{2}\operatorname{Tr}(\mathbf{X}^{\top}\mathbf{L}\mathbf{X})$, we have $\nabla_{\mathbf{X}}\mathcal{E}(\mathbf{X})=\mathbf{L}\mathbf{X}$. Therefore, the diffusion equation $\frac{\partial\mathbf{X}}{\partial t}=-\mathbf{L}\mathbf{X}=-\nabla_{\mathbf{X}}\mathcal{E}(\mathbf{X})$ is its negative gradient flow, which monotonically decreases the Dirichlet energy toward its minimum.}:
\begin{equation}\label{Dirichlet energy}
	\min \mathcal{E}(\mathcal{X}; \mathcal{G}) = \frac{1}{2} \sum_{(v, i)} \sum_{(u, j)} \mathbf{S}^{(v, u)}_{i, j} \Vert \mathbf{x}_{i}^{(v)} - \mathbf{x}_{j}^{(u)} \Vert^2 .
\end{equation}
Here, the energy functional $\mathcal{E}(\cdot, \cdot)$ characterizes the global smoothness of the multi-view feature space.
By penalizing discrepancies between instance features, it adaptively assigns coupling weights $\mathbf{S}$ to drive the system toward a global consensus, where individual fibers are aligned through the diffusion flow.
Crucially, this minimization objective reveals that the system encourages diffusion among similar instances while suppressing it between dissimilar ones, which directly aligns with the \textit{homophily principle} \cite{mcpherson2001birds}.
However, for a complete graph, there always exist feature pairs with significant discrepancies that are assigned near-zero diffusivity coefficients.
Consequently, $\mathcal{G}$ inevitably introduces noise and deviates markedly from the desired \textit{oracle} topology $\mathcal{G}^*$\footnote{The oracle graph $\mathcal{G}^*$ denotes an ideal, label-driven topology where an edge exists between two instances if and only if they share the same label, and is zero otherwise.}.
Such theoretical insight motivates us to identify and retain informative edges, purging non-essential noise, narrowing the gap between the sampled topology and the oracle graph, and improving diffusion efficiency:
\begin{equation}\label{truncation}
	\!\! \hat{\mathbf{S}} = \mathbf{S} \odot \mathbf{M}(\theta), \quad \mathbf{M}^{(v, u)}_{i, j}(\theta) \!=\! \begin{cases}
		1, & \!\!\!s(\mathbf{x}_{i}^{(v)}, \mathbf{x}_{j}^{(u)}) \ge \theta, \\
		0, & \!\!\!       s(\mathbf{x}_{i}^{(v)}, \mathbf{x}_{j}^{(u)}) < \theta. \end{cases}
\end{equation}
Here, $\mathbf{M}(\theta)$ denotes the sparsity-indicating matrix, and $s(\cdot, \cdot)$ acts as a similarity measure, which can be easily obtained from typical \textit{geometric priors}, such as Euclidean distance or cosine similarity.
By performing diffusion flow modeling on the resulting sparse $\hat{\mathbf{S}}$, the space complexity is reduced from $O(N^2V^2)$ to $O(E)$, where $E$ denotes the number of edges after truncation.
Essentially, the truncation process in Eq.~\eqref{truncation}, which minimizes the Dirichlet energy in Eq.~\eqref{Dirichlet energy}, corresponds to maximum likelihood estimation within a \textit{Gaussian Markov Random Field} \cite{siden2020deep}.
From this statistical perspective, we can further derive an upper bound on the structural consistency error between the truncated adjacency matrix $\hat{S}$ and the oracle adjacency matrix $S^*$ derived from $\mathcal{G}^*$.

\begin{theorem}\label{theo:consistency}
	(Support Recovery Guarantee)
	Let $\mathbf{S}^*$ denote the adjacency matrix induced by the oracle topology $\mathcal{G}^*$, and let $\hat{\mathbf{S}}$ be the truncated adjacency matrix obtained via Eq.~\eqref{truncation}.
	Under the incoherence condition\footnote{In practice, the incoherence condition is generally satisfied, as the intrinsic semantic features of distinct categories tend to be sufficiently decorrelated.},
	the probability of structural mismatch between $\hat{\mathbf{S}}$ and $\mathbf{S}^*$ is bounded as
	\begin{equation}
		\mathbb{P}\left( \operatorname{supp}(\hat{\mathbf{S}}) \neq \operatorname{supp}(\mathbf{S}^*) \right) \le (NV)^2 \exp\left( - C D \theta^2 \right).
	\end{equation}
\end{theorem}
\noindent
The proof is deferred to \textbf{Appendix~\ref{app:proof-thm1}}.
Here, $C>0$ is a constant, $D$ denotes the feature dimensionality, $\theta$ characterizes the discriminative gap of the similarity measure in Eq.~\eqref{truncation}, and $\operatorname{supp}(\cdot)$ denotes the support set defining the connectivity of the induced topology.
Theorem \ref{theo:consistency} indicates that as the feature dimensionality $D$ increases, the probability of structural mismatch decays exponentially, thereby providing strong theoretical justification and practical performance guarantees for our sampling strategy.

\subsection{Efficient Training with Ego-Net Style Architecture}
According to the synchronous diffusion mechanism in Eq.~\eqref{global diffusion}, a dual-traversal over the multi-view instance space is required, which incurs prohibitive computational overhead. To address this bottleneck, we develop a centralized fusion architecture. Specifically, we introduce virtual anchors to bridge the features of individual instances across different views:
\begin{equation}\label{egonet}
	\mathcal{T}_i = \{ (\mathbf{h}_i, \mathbf{x}_j^{(v)}) \mid v \in {1, \dots, V}, j \in \mathcal{N}_i^{(v)} \},
\end{equation}
where $\mathbf{h}_i$ is the virtual anchor acts as a proxy to aggregate and relay information across views,
$\mathcal{N}_i^{(v)}$ denotes the neighbors for instance $i$ within view $v$, which is governed by the sparsity-indicating matrix $\mathbf{M}(\theta)$ defined in Eq.~\eqref{truncation},
and $\mathcal{T}_i$ comprises the centralized connectivity pattern that facilitates cross-view information exchange.
This can be viewed as the construction of an \textit{Ego-Net} \cite{hamilton2017inductive} style architecture for each instance.
Through Eq.~(\ref{egonet}), features of any instance across different views are flowed through virtual anchors, transforming the time-consuming cross-view traversal interactions into a centralized paradigm.
Consequently, such strategy allows us to train the model efficiently with \textit{linear} time complexity relative to the number of instances $N$ by using batch processing.
In this way, the synchronized multi-view diffusion process described by Eq.~(\ref{global diffusion}) can be expressed in the following form:

\begin{equation}
	\frac{\partial \mathbf{h}_i(t)}{\partial t} =
	\sum_{v=1}^{V} \alpha ^{(v)}(t)
	\sum_{j \in \mathcal{N}_i^{(v)}}
	\mathcal{T}^{(v)}_{i,j}
	(\mathbf{x}^{(v)}_j(t) - \mathbf{h}_i(t)),
\end{equation}
where $\mathcal{T}^{(v)}_{i,j}$ denotes the diffusion intensity between neighbor $j$ in $v$-th view and the virtual anchor $\mathbf{h}_i$.
Similar to Eq.~(\ref{forward}), we can obtain the final forward diffusion form:
\begin{equation}\label{centralized diffusion}
	\begin{aligned}
		\mathbf{h}_i(t+1) &= \left[1 - \tau \sum_{v=1}^{V} \alpha^{(v)}(t) \sum_{j \in \mathcal{N}_i^{(v)}} \mathcal{T}^{(v)}_{i,j}(t)\right] \mathbf{h}_i(t)\\
		&+ \tau \sum_{v=1}^{V} \alpha^{(v)}(t) \sum_{j \in \mathcal{N}_i^{(v)}} \mathcal{T}^{(v)}_{i,j}(t) \mathbf{x}_j^{(v)}(t).
	\end{aligned}
\end{equation}
Through Eq.~\eqref{centralized diffusion}, the features of each instance across diverse views propagate through the virtual anchors adaptively.\footnote{After the diffusion process, we directly feed the resulting virtual-anchor representations $\mathbf{h}$ into a simple linear classifier to obtain class predictions.}
More importantly, this framework enables the proposed synchronous diffusion mechanism to achieve efficient and scalable information propagation in large-scale scenarios, as demonstrated in subsequent experiments.
The algorithmic procedure of the proposed framework is provided in \textbf{Appendix~\ref{app:algorithm}}.

\section{Empirical Studies}
We evaluate the proposed \method across diverse scenarios, including multi-omics cancer subtyping, heterogeneous graph node classification, and general multi-view classification.
All experiments are repeated five times with different
random seeds and their means and standard deviations are reported.
~\\

\noindent \textbf{Datasets.}
To comprehensively assess the generalization capability of \method, we
conduct extensive evaluations across three representative multi-view
learning scenarios.
Specifically, for multi-omics cancer subtyping, we employ five widely used
cancer cohorts: BRCA, LGG, UCEC, GBMLGG, and TCGA.
For heterogeneous graph node classification, our evaluation incorporates
several well-established benchmarks, including FreeBase, DBLP, IMDB, Yelp,
and AMiner.
For general multi-view classification, we utilize eight commonly used
datasets: Scene15, YouTube, MITIndoor, HW, IAPR, Animals, Caltech, and
ESPGame.
Additionally, we examine the scalability of \method in large-scale settings
using four datasets: NoisyMNIST, YTF, CIFAR-10, and VGGFace.
Detailed statistics are summarized in \textbf{Appendix~\ref{app:datasets}}.
~\\

\begin{table*}[ht]
	\centering
	\caption{Classification results (mean\% $\pm$ std\%) on cancer subtype datasets. Best and second-best results are highlighted in red and blue, respectively. ``OOM'' indicates out-of-memory errors.}
	\label{tab:multi-omics_results}
	\normalsize 
	\setlength{\tabcolsep}{5pt}        
	\renewcommand{\arraystretch}{1.16} 
	\begin{tabular}{lcccccccccc}
		\toprule
		\multirow{2}{*}{Method} & \multicolumn{2}{c}{BRCA} & \multicolumn{2}{c}{LGG} & \multicolumn{2}{c}{UCEC} & \multicolumn{2}{c}{GBMLGG} & \multicolumn{2}{c}{TCGA} \\
		& Macro-F1 & Micro-F1 & Macro-F1 & Micro-F1 & Macro-F1 & Micro-F1 & Macro-F1 & Micro-F1 & Macro-F1 & Micro-F1 \\
		\midrule
		SVM     & 43.6 $\pm$ 0.0 & 70.2 $\pm$ 0.0 & 60.0 $\pm$ 0.0 & 60.0 $\pm$ 0.0 & 28.0 $\pm$ 0.0 & 72.5 $\pm$ 0.0 & 39.4 $\pm$ 0.0 & 52.0 $\pm$ 0.0 & 61.4 $\pm$ 0.0 & 68.0 $\pm$ 0.0 \\
		RF      & 68.8 $\pm$ 0.0 & 80.2 $\pm$ 0.0 & 49.6 $\pm$ 0.0 & 49.8 $\pm$ 0.0 & 28.7 $\pm$ 0.0 & 72.5 $\pm$ 0.0 & 47.0 $\pm$ 0.0 & \cellcolor{second}54.6 $\pm$ 0.0 & 52.0 $\pm$ 0.0 & 67.0 $\pm$ 0.0 \\
		DeepMO  & \cellcolor{second}76.4 $\pm$ 4.9 & \cellcolor{second}81.4 $\pm$ 2.1 & \cellcolor{best}63.7 $\pm$ 3.8 & \cellcolor{best}64.2 $\pm$ 3.1 & \cellcolor{second}53.8 $\pm$ 1.9 & 80.2 $\pm$ 3.1 & \cellcolor{second}51.1 $\pm$ 2.2 & 53.5 $\pm$ 2.4 & 64.6 $\pm$ 6.9 & 73.0 $\pm$ 6.9 \\
		MOGONET & 58.9 $\pm$ 2.6 & 71.6 $\pm$ 1.5 & 61.8 $\pm$ 2.4 & 62.3 $\pm$ 1.8 & 43.7 $\pm$ 0.6 & 75.4 $\pm$ 1.9 & 43.6 $\pm$ 2.7 & 49.2 $\pm$ 1.5 & 38.5 $\pm$ 0.4 & 42.1 $\pm$ 0.7 \\
		MoGCN   & 55.3 $\pm$ 0.7 & 73.6 $\pm$ 0.4 & 33.8 $\pm$ 0.0 & 51.1 $\pm$ 0.0 & 28.0 $\pm$ 0.0 & 71.6 $\pm$ 0.0 & 41.5 $\pm$ 8.4 & 53.0 $\pm$ 6.3 & \cellcolor{second}66.7 $\pm$ 0.4 & \cellcolor{second}73.9 $\pm$ 0.4 \\
		Moanna  & 69.9 $\pm$ 2.0 & 77.5 $\pm$ 1.8 & 60.8 $\pm$ 1.9 & 61.3 $\pm$ 1.8 & 53.4 $\pm$ 2.0 & 81.6 $\pm$ 2.2 & 48.1 $\pm$ 3.0 & 53.1 $\pm$ 4.9 & 51.7 $\pm$ 2.3 & 58.7 $\pm$ 2.2 \\
		MOSGAT  & 57.8 $\pm$ 5.3 & 73.0 $\pm$ 3.2 & 55.0 $\pm$ 2.9 & 58.5 $\pm$ 1.3 & 52.3 $\pm$ 4.1 & \cellcolor{second}82.1 $\pm$ 1.9 & 46.7 $\pm$ 1.5 & 49.6 $\pm$ 1.2 & OOM & OOM \\
		\method & \cellcolor{best}81.4 $\pm$ 1.5 & \cellcolor{best}85.2 $\pm$ 0.8 & \cellcolor{second}62.5 $\pm$ 0.6 & \cellcolor{second}63.1 $\pm$ 0.7 & \cellcolor{best}55.0 $\pm$ 0.4 & \cellcolor{best}85.2 $\pm$ 0.5 & \cellcolor{best}52.7 $\pm$ 1.1 & \cellcolor{best}56.1 $\pm$ 1.0 & \cellcolor{best}70.7 $\pm$ 0.5 & \cellcolor{best}78.0 $\pm$ 0.3 \\
		\bottomrule
	\end{tabular}
\end{table*}

\begin{table*}[ht]
	\centering
	\caption{Node classification results (mean\% $\pm$ std\%) on heterogeneous graph datasets. Best and second-best results are highlighted in red and blue, respectively.}
	\label{tab:hetero_results}
	\normalsize 
	\setlength{\tabcolsep}{5pt}        
	\renewcommand{\arraystretch}{1.16} 
	\begin{tabular}{lcccccccccc}
		\toprule
		\multirow{2}{*}{Method} & \multicolumn{2}{c}{FreeBase} & \multicolumn{2}{c}{DBLP} & \multicolumn{2}{c}{IMDB} & \multicolumn{2}{c}{Yelp} & \multicolumn{2}{c}{AMiner} \\
		& Macro-F1 & Micro-F1 & Macro-F1 & Micro-F1 & Macro-F1 & Micro-F1 & Macro-F1 & Micro-F1 & Macro-F1 & Micro-F1 \\
		\midrule
		GCN    & 44.6 $\pm$ 1.4 & 37.4 $\pm$ 1.1 & 90.1 $\pm$ 0.8 & 91.6 $\pm$ 0.6 & 24.3 $\pm$ 0.2 & 55.4 $\pm$ 0.2 & 52.0 $\pm$ 0.2 & 67.4 $\pm$ 0.9 & 68.4 $\pm$ 0.6 & 81.6 $\pm$ 0.7 \\
		HAN    & 62.1 $\pm$ 2.4 & 48.8 $\pm$ 3.4 & 89.3 $\pm$ 0.4 & 90.4 $\pm$ 0.4 & 23.9 $\pm$ 0.5 & 55.9 $\pm$ 0.8 & 48.3 $\pm$ 0.3 & 48.9 $\pm$ 0.6 & 72.3 $\pm$ 0.6 & 84.8 $\pm$ 0.1 \\
		DMGI   & 54.8 $\pm$ 2.1 & 41.1 $\pm$ 1.9 & 65.7 $\pm$ 0.2 & 71.1 $\pm$ 1.0 & 35.3 $\pm$ 1.0 & 57.3 $\pm$ 0.8 & 51.6 $\pm$ 0.4 & 69.8 $\pm$ 0.2 & 30.3 $\pm$ 0.7 & 65.5 $\pm$ 0.5 \\
		IGNN   & \cellcolor{second}65.1 $\pm$ 0.1 & \cellcolor{second}61.7 $\pm$ 0.2 & 86.8 $\pm$ 0.1 & 87.5 $\pm$ 0.9 & 45.3 $\pm$ 0.3 & 54.8 $\pm$ 0.7 & 64.5 $\pm$ 0.4 & 71.2 $\pm$ 0.6 & 74.5 $\pm$ 0.6 & \cellcolor{second}85.2 $\pm$ 0.3 \\
		MRGCN  & 57.0 $\pm$ 0.3 & 53.9 $\pm$ 0.1 & 89.5 $\pm$ 0.3 & 90.5 $\pm$ 0.6 & 45.2 $\pm$ 0.6 & 47.7 $\pm$ 0.7 & 54.3 $\pm$ 0.4 & 73.7 $\pm$ 0.4 & 73.4 $\pm$ 0.4 & 82.9 $\pm$ 0.4 \\
		SSDCM  & 53.9 $\pm$ 2.9 & 53.9 $\pm$ 2.9 & 89.4 $\pm$ 0.5 & 89.9 $\pm$ 0.9 & 49.4 $\pm$ 0.2 & 59.1 $\pm$ 0.6 & 52.7 $\pm$ 0.8 & 70.2 $\pm$ 0.5 & 26.2 $\pm$ 0.5 & 55.5 $\pm$ 0.8 \\
		MHGCN  & 62.3 $\pm$ 1.1 & 54.4 $\pm$ 1.3 & \cellcolor{second}93.0 $\pm$ 0.4 & \cellcolor{second}93.6 $\pm$ 0.5 & \cellcolor{second}51.5 $\pm$ 0.7 & \cellcolor{second}64.2 $\pm$ 0.3 & 60.9 $\pm$ 0.6 & 73.3 $\pm$ 0.3 & 75.0 $\pm$ 0.5 & 85.2 $\pm$ 0.4 \\
		AMOGCN & 63.6 $\pm$ 0.1 & 49.2 $\pm$ 0.6 & 92.3 $\pm$ 0.3 & 92.8 $\pm$ 0.5 & 50.2 $\pm$ 0.5 & \cellcolor{best}65.1 $\pm$ 0.4 & \cellcolor{second}67.5 $\pm$ 0.6 & 67.5 $\pm$ 0.5 & \cellcolor{second}75.7 $\pm$ 0.4 & 83.4 $\pm$ 0.5 \\
		HMGE   & 62.2 $\pm$ 0.1 & 45.8 $\pm$ 0.1 & 91.6 $\pm$ 0.4 & 92.4 $\pm$ 0.3 & 32.9 $\pm$ 0.3 & 57.4 $\pm$ 0.5 & 58.6 $\pm$ 0.3 & \cellcolor{second}74.8 $\pm$ 0.3 & 73.7 $\pm$ 0.6 & 84.8 $\pm$ 0.3 \\
		\method & \cellcolor{best}66.2 $\pm$ 0.8 & \cellcolor{best}71.3 $\pm$ 0.4
		& \cellcolor{best}93.4 $\pm$ 0.2 & \cellcolor{best}93.8 $\pm$ 0.2
		& \cellcolor{best}53.0 $\pm$ 0.6 & 62.2 $\pm$ 1.0
		& \cellcolor{best}77.7 $\pm$ 0.5 & \cellcolor{best}81.0 $\pm$ 0.6
		& \cellcolor{best}76.2 $\pm$ 0.6 & \cellcolor{best}86.2 $\pm$ 0.2 \\
		\bottomrule
	\end{tabular}
\end{table*}

\noindent \textbf{Compared Methods.}
For multi-omics cancer subtyping, we compare \method against several state-of-the-art baselines, including SVM, Random Forest, DeepMO \cite{lin2020classifying}, MOGONET \cite{wang2021mogonet}, MoGCN \cite{li2022mogcn}, Moanna \cite{lupat2023moanna}, and MOSGAT \cite{wu2024mosgat}.
For heterogeneous graph learning, we benchmark with GCN \cite{kipf2017semi}, HAN \cite{wang2019heterogeneous}, DMGI \cite{park2020unsupervised}, IGNN \cite{gu2020implicit}, MRGCN \cite{huang2020mr}, SSDCM \cite{mitra2021semi}, MHGCN \cite{yu2022multiplex}, AMOGCN \cite{chen2024attributed}, and HMGE \cite{abdous2023hierarchical}.
For general multi-view scenarios, we evaluate with Co-GCN \cite{li2020co}, PDMF \cite{xu2023progressive}, LGCN-FF \cite{chen2023learnable}, ECMGD \cite{lu2024towards}, TUNED \cite{huang2025trusted}, KAMSSM \cite{liao2026static}, and CoGFormer \cite{Lai_Li_Wang_Wang_2026}.
Detailed descriptions of all baseline methods are provided in \textbf{Appendix~\ref{app:baselines}}.
~\\

\noindent \textbf{Experimental Settings.}
All compared methods use the default parameters as specified in their original papers.
For the proposed \method, we adopt fixed network hyperparameters.
Specifically, we use a $512$-dimensional linear layer to align the dimensions of the multi-view data.
The number of attention heads is set to $4$, with the attention layer dimension set to $128$ and the feed-forward network dimension set to $256$.
Dropout is set to $0.2$, weight decay is $1e-5$, and the batch size is fixed at $128$ to ensure training can be conducted on most single GPUs.
We set the maximum number of training epochs to $200$ and adopt the model parameters corresponding to the best validation performance for evaluation.
During training, we use the Adam optimizer and the learning rate is set to a default value of $5e-4$.
Further details of the experiment can be found in \textbf{Appendix~\ref{app:implementation}}.

\subsection{Performance}
\noindent \textbf{Multi-omics Cancer Subtyping.}
Table \ref{tab:multi-omics_results} summarizes the classification results
on cancer subtype datasets, with 10\% of the labeled samples used for
training.
The proposed
\method consistently outperforms the baselines in terms of both Macro-F1
and Micro-F1 on BRCA, UCEC, GBMLGG, and TCGA, while achieving the
second-best performance on LGG.
These results demonstrate the effectiveness of \method in modeling complex
multi-omics dependencies and highlight its potential for cancer subtype
analysis.

~\\

\noindent \textbf{Heterogeneous Graph Node Classification.}
In the node classification setting on heterogeneous graphs, 20\% of the labeled nodes were used for training, with the remaining data split evenly between validation and testing. The quantitative results of the experiments are presented in Table \ref{tab:hetero_results}. As reported, \method consistently outperforms the baselines in both metrics across the FreeBase, DBLP, Yelp, and AMiner datasets, and achieves the best Macro-F1 score on the IMDB dataset.
These results further confirm the superiority of the proposed \method in handling real-world heterogeneous data and relationships.
~\\

\begin{table*}[htbp]
	\centering
	\caption{Classification results (mean\% $\pm$ std\%) on multi-view datasets. Best and second-best results are highlighted in red and blue, respectively.}
	\label{tab:Mv_results}
	\normalsize 
	\setlength{\tabcolsep}{5pt}        
	\renewcommand{\arraystretch}{1.17} 
	\begin{tabular}{clcccccccc}
		\toprule
		Metrics & Methods & Co-GCN & PDMF & LGCN-FF & ECMGD & TUNED & KAMSSM & CoGFormer & SynMDiff \\
		\midrule
		\multirow{8}{*}{Macro-F1}
		& Scene15   & 18.9 $\pm$ 8.7 & 39.8 $\pm$ 4.6 & 42.3 $\pm$ 5.7 & 69.3 $\pm$ 4.3 & 70.0 $\pm$ 3.0 & 67.8 $\pm$ 0.8 & \cellcolor{second}74.0 $\pm$ 0.7 & \cellcolor{best}81.1 $\pm$ 0.3 \\
		& YouTube   & 43.4 $\pm$ 4.0 & 36.9 $\pm$ 3.3 & 42.3 $\pm$ 5.7 & 59.0 $\pm$ 0.4 & 57.3 $\pm$ 0.9 & 56.5 $\pm$ 1.1 & \cellcolor{second}63.6 $\pm$ 1.4 & \cellcolor{best}70.5 $\pm$ 0.5 \\
		& MITIndoor & \cellcolor{second}51.8 $\pm$ 0.8 & 48.9 $\pm$ 0.3 & 21.1 $\pm$ 7.8 & 36.5 $\pm$ 8.1 & 22.4 $\pm$ 3.4 & 25.8 $\pm$ 3.0 & 51.2 $\pm$ 2.3 & \cellcolor{best}55.2 $\pm$ 0.8 \\
		& HW        & 94.9 $\pm$ 2.0 & 90.0 $\pm$ 2.3 & 91.5 $\pm$ 2.8 & 95.4 $\pm$ 0.0 & 88.9 $\pm$ 1.6 & 96.3 $\pm$ 0.2 & \cellcolor{second}96.6 $\pm$ 0.2 & \cellcolor{best}97.2 $\pm$ 0.2 \\
		& IAPR      & 55.8 $\pm$ 3.8 & 60.1 $\pm$ 0.7 & 57.0 $\pm$ 1.4 & 65.5 $\pm$ 0.2 & 64.1 $\pm$ 4.4 & 65.9 $\pm$ 0.2 & \cellcolor{second}66.2 $\pm$ 0.2 & \cellcolor{best}70.2 $\pm$ 0.1 \\
		& Animals   & 61.9 $\pm$ 4.3 & 70.6 $\pm$ 0.3 & 62.9 $\pm$ 6.2 & 74.8 $\pm$ 0.4 & 74.7$\pm$ 0.6 & 70.9 $\pm$ 0.6 & \cellcolor{second}77.3 $\pm$ 0.2 & \cellcolor{best}79.1 $\pm$ 0.2 \\
		& Caltech   & 9.8 $\pm$ 1.4  & 15.3 $\pm$ 0.7 & 33.4 $\pm$ 0.5 & 35.3 $\pm$ 0.4 & 32.9 $\pm$ 0.8 & 27.8 $\pm$ 0.4 & \cellcolor{second}37.6 $\pm$ 0.4 & \cellcolor{best}38.2 $\pm$ 0.4 \\
		& ESPGame   & 67.3 $\pm$ 2.3 & 81.1 $\pm$ 0.8 & 68.7 $\pm$ 0.4 & \cellcolor{second}82.7 $\pm$ 0.1 & 77.6 $\pm$ 1.6 & 82.5 $\pm$ 1.4 & 77.8 $\pm$ 0.3 & \cellcolor{best}83.9 $\pm$ 0.6 \\
		\midrule
		\multirow{8}{*}{Micro-F1}
		& Scene15   & 31.5 $\pm$ 7.7 & 39.8 $\pm$ 4.6 & 50.1 $\pm$ 4.4 & 75.3 $\pm$ 0.4 & 72.9 $\pm$ 1.9 & 68.8 $\pm$ 0.8 & \cellcolor{second}75.5 $\pm$ 0.6 & \cellcolor{best}81.0 $\pm$ 0.4 \\
		& YouTube   & 45.1 $\pm$ 2.3 & 36.9 $\pm$ 3.3 & 47.3 $\pm$ 1.8 & 59.4 $\pm$ 0.4 & 59.0 $\pm$ 0.7 & 56.9 $\pm$ 1.1 & \cellcolor{second}63.8 $\pm$ 1.2 & \cellcolor{best}70.7 $\pm$ 0.4 \\
		& MITIndoor & \cellcolor{second}52.9 $\pm$ 0.8 & 48.2 $\pm$ 0.2 & 23.1 $\pm$ 7.3 & 38.2 $\pm$ 6.4 & 30.0 $\pm$ 3.8 & 25.5 $\pm$ 3.0 & 52.5 $\pm$ 1.7 & \cellcolor{best}56.8 $\pm$ 0.7 \\
		& HW        & 95.0 $\pm$ 2.0 & 90.0 $\pm$ 2.3 & 92.6 $\pm$ 0.1 & 95.6 $\pm$ 0.4 & 88.9 $\pm$ 1.6 & 96.3 $\pm$ 0.2 & \cellcolor{second}96.6 $\pm$ 0.2 & \cellcolor{best}97.2 $\pm$ 0.2 \\
		& IAPR      & 54.8 $\pm$ 3.9 & 57.5 $\pm$ 0.8 & 55.8 $\pm$ 1.0 & 64.5 $\pm$ 0.2 & 63.6 $\pm$ 3.6 & \cellcolor{second}64.9 $\pm$ 0.2 & 64.8 $\pm$ 0.2 & \cellcolor{best}69.2 $\pm$ 0.1 \\
		& Animals   & 71.2 $\pm$ 3.6 & 78.9 $\pm$ 0.2 & 74.2 $\pm$ 4.1 & 81.1 $\pm$ 0.4 & 81.3 $\pm$ 0.3 & 77.2 $\pm$ 0.5 & \cellcolor{second}83.4 $\pm$ 0.1 & \cellcolor{best}84.2 $\pm$ 0.2 \\
		& Caltech   & 34.3 $\pm$ 0.9 & 15.3 $\pm$ 0.7 & 40.2 $\pm$ 0.8 & 54.4 $\pm$ 0.3 & 54.4 $\pm$ 0.6 & 48.0 $\pm$ 0.5 & \cellcolor{second}56.5 $\pm$ 0.2 & \cellcolor{best}58.0 $\pm$ 0.3 \\
		& ESPGame   & 67.4 $\pm$ 2.4 & 81.2 $\pm$ 0.8 & 68.8 $\pm$ 0.4 & 82.8 $\pm$ 0.1 & 78.4 $\pm$ 1.5 & \cellcolor{second}82.9 $\pm$ 1.4 & 77.9 $\pm$ 0.3 & \cellcolor{best}84.2 $\pm$ 0.5 \\
		\bottomrule
	\end{tabular}
\end{table*}

\begin{table*}[htbp]
	\centering
	\caption{
		Classification results (mean\% $\pm$ std\%), inference time (s), and memory usage (MB) on large-scale datasets. The best and second-best results are highlighted in red and blue, respectively. ``OOM'' indicates out-of-memory errors.
	}
	\label{tab:large_results}
	\footnotesize 
	\setlength{\tabcolsep}{3pt}        
	\renewcommand{\arraystretch}{1.3} 
	\begin{tabular}{lcccccccccccccccc}
		\toprule
		Dataset & \multicolumn{4}{c}{NoisyMNIST} & \multicolumn{4}{c}{YTF} & \multicolumn{4}{c}{CIFAR-10} & \multicolumn{4}{c}{VGGFace} \\
		Samples & \multicolumn{4}{c}{70,000} & \multicolumn{4}{c}{286,006} & \multicolumn{4}{c}{50,000} & \multicolumn{4}{c}{34,027} \\
		\midrule
		Metric & Macro-F1 & Micro-F1 & Time & Mem & Macro-F1 & Micro-F1 & Time & Mem & Macro-F1 & Micro-F1 & Time & Mem & Macro-F1 & Micro-F1 & Time & Mem \\
		\midrule
		Co-GCN    & 31.5 $\pm$ 7.7 & 31.5 $\pm$ 7.7 & 5.0 & 128.0 & 85.5 $\pm$ 0.2 & 88.2 $\pm$ 0.2 & 37.7 & 201.8 & 97.2 $\pm$ 0.4 & 97.2 $\pm$ 0.4 & 7.4 & 216.0 & 32.0 $\pm$ 0.4 & 32.8 $\pm$ 0.4 & 4.5 & 196.3 \\
		PDMF      & \cellcolor{second}94.1 $\pm$ 0.7 & \cellcolor{second}94.2 $\pm$ 0.7 & 2.3 & 2834.2 & 55.8 $\pm$ 0.3 & 60.6 $\pm$ 0.2 & 11.2 & 3544.9 & 90.9 $\pm$ 0.0 & 90.9 $\pm$ 0.0 & 1.8 & 2973.6 & 47.0 $\pm$ 0.4 & 46.4 $\pm$ 0.4 & 11.1 & 2949.1 \\
		LGCNFF    & OOM & OOM & - & - & OOM & OOM & - & - & OOM & OOM & - & - & OOM & OOM & - & - \\
		ECMGD     & OOM & OOM & - & - & OOM & OOM & - & - & OOM & OOM & - & - & OOM & OOM & - & - \\
		CoGFormer & OOM & OOM & - & - & OOM & OOM & - & - & OOM & OOM & - & - & OOM & OOM & - & - \\
		TUNED     & 86.5 $\pm$ 1.7 & 86.8 $\pm$ 1.7 & 10.6 & 42.1 & \cellcolor{second}98.6 $\pm$ 0.1 & \cellcolor{second}98.9 $\pm$ 0.1 & 91.9 & 76.1 & 98.9 $\pm$ 0.1 & 98.9 $\pm$ 0.1 & 12.4 & 124.6 & \cellcolor{second}50.9 $\pm$ 0.2 & \cellcolor{second}51.5 $\pm$ 0.2 & 9.8 & 60.1 \\
		KAMSSM    & 79.3 $\pm$ 0.7 & 79.6 $\pm$ 0.3 & 0.4 & 1286.2 & 78.4 $\pm$ 3.6 & 82.2 $\pm$ 2.8 & 2.0 & 10049.7 & \cellcolor{second}99.1 $\pm$ 0.1 & \cellcolor{second}99.1 $\pm$ 0.1 & 0.4 & 1610.1 & 33.4 $\pm$ 0.8 & 33.6 $\pm$ 0.7 & 0.3 & 999.0 \\
		\method   & \cellcolor{best}98.0 $\pm$ 0.0 & \cellcolor{best}98.0 $\pm$ 0.0 & 0.7 & 877.6 & \cellcolor{best}100.0 $\pm$ 0.0 & \cellcolor{best}100.0 $\pm$ 0.0 & 3.6 & 5886.0 & \cellcolor{best}99.4 $\pm$ 0.0 & \cellcolor{best}99.4 $\pm$ 0.0 & 0.6 & 1428.0 & \cellcolor{best}57.3 $\pm$ 0.3 & \cellcolor{best}57.4 $\pm$ 0.3 & 0.5 & 747.4 \\
		\bottomrule
	\end{tabular}
\end{table*}

\noindent \textbf{Multi-view Classification.}
For the multi-view classification setting, we use 10\% of the labeled data for training, 10\% for validation, and the remaining data for testing. As shown in Table \ref{tab:Mv_results}, our method outperforms a range of multi-view fusion baselines across all datasets.
This demonstrates the significant superiority of our synchronized diffusion mechanism in capturing the complex relationships inherent in multi-view data.
Notably, on the YouTube, IAPR, and ESPGame datasets, which involve multi-modal representations, our method also shows clear performance gains.
~\\

\noindent \textbf{Large-Scale Scenarios.}
To validate the efficiency and scalability of our method during both training and inference, which are critical for real-world deployment, we conduct experiments on four large-scale datasets. The results, inference time, and memory usage across different methods are reported in Table~\ref{tab:large_results}.
Many advanced multi-view fusion methods require storing the complete topological structure, making them difficult to scale and often leading to out-of-memory errors on large datasets. Although methods such as KAMSSM and TUNED can be applied in these settings, they underutilize multi-view information, resulting in suboptimal performance.
In contrast, the proposed \method remains scalable, achieving the best performance across all four large-scale datasets while incurring lower inference latency and memory consumption, demonstrating its practical suitability for resource-constrained and large-scale real-world applications.

\subsection{Visualization}
For a more intuitive comparison, Figure~\ref{fig:TSNE} presents t-SNE visualizations of the embeddings learned by different methods on three representative datasets.
In addition, we quantitatively evaluate the results using homogeneity and completeness scores \cite{peng2021maximum}.
Across all settings, the embeddings learned by \method demonstrate more compact clusters and clearer inter-class separation compared to the baselines, and achieve the highest scores on both homogeneity and completeness.
This advantage is more pronounced on the TCGA dataset, which involves a larger number of samples and classes, thus posing a more challenging task.
These visualizations further demonstrate that, due to its ability to capture synchronized cross-view interactions, \method learns more discriminative representations in complex multi-view scenarios.

\begin{figure}[!t]
	\centering
	\includegraphics[width=1.0\linewidth]{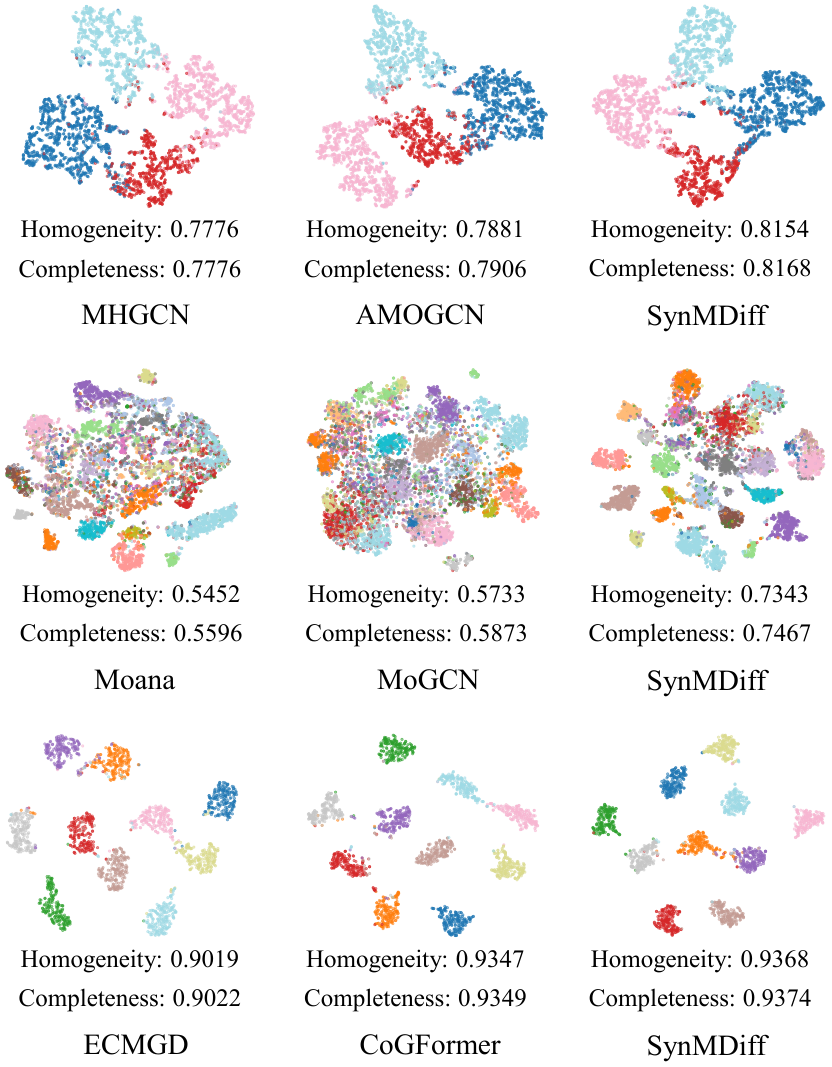}
	\Description{
		A three-by-three grid of t-SNE scatter plots comparing learned
		embeddings. The rows correspond to DBLP, TCGA, and HW, respectively,
		while the columns compare two baseline methods with SynMDiff.
		Points belonging to different classes are distinguished by color.
		On DBLP, SynMDiff produces more compact and clearly separated class
		clusters than MHGCN and AMOGCN. On TCGA, the baseline embeddings
		exhibit substantial overlap, whereas SynMDiff separates most classes
		into distinct clusters. On HW, all methods form relatively clear
		clusters, with SynMDiff achieving slightly better separation.
		SynMDiff obtains the highest homogeneity and completeness scores in
		all three rows: 0.8154 and 0.8168 on DBLP, 0.7343 and 0.7467 on TCGA,
		and 0.9368 and 0.9374 on HW.
	}
	\caption{t-SNE visualizations of embeddings learned by different methods on the heterogeneous graph dataset DBLP (top row), the multi-omics dataset TCGA (middle row), and the multi-view dataset HW (bottom row).}
	\label{fig:TSNE}
\end{figure}

\subsection{Ablation Study}\label{sec:ablation}
In the ablation study, we aim to clarify two questions: (i) How important are intra-view and inter-view propagation to model performance?
(ii) Without the sampling strategy, does the model benefit from a broader receptive field, or does its performance degrade due to the noise introduced by dense adjacency relations?
To answer the first question, we derive two variants that retain only intra-view or inter-view information propagation, denoted as w/o Inter and w/o Intra, respectively.
To answer the second, we remove the sampling module and the centralized training architecture, allowing dense interactions between arbitrary pairs of features across multiple views, this variant is denoted as w/o Samp.
We evaluate these three variants together with \method on six datasets under different multi-view settings. The results are shown in Figure~\ref{fig:ablation}.
As can be clearly seen, all three variants consistently underperform \method across all datasets. Somewhat surprisingly, on most datasets, inter-view information propagation appears to be more important than intra-view aggregation, which further highlights the significance of our study. In addition, removing the sampling strategy leads to a substantial performance drop, suggesting that densely modeling interactions among multi-view features is not a desirable choice.
Such dense interactions introduce a large number of semantically irrelevant connections, and the resulting propagation of noise ultimately degrades model performance.

\begin{figure}[!t]
	\centering
	\includegraphics[width=0.95\linewidth]{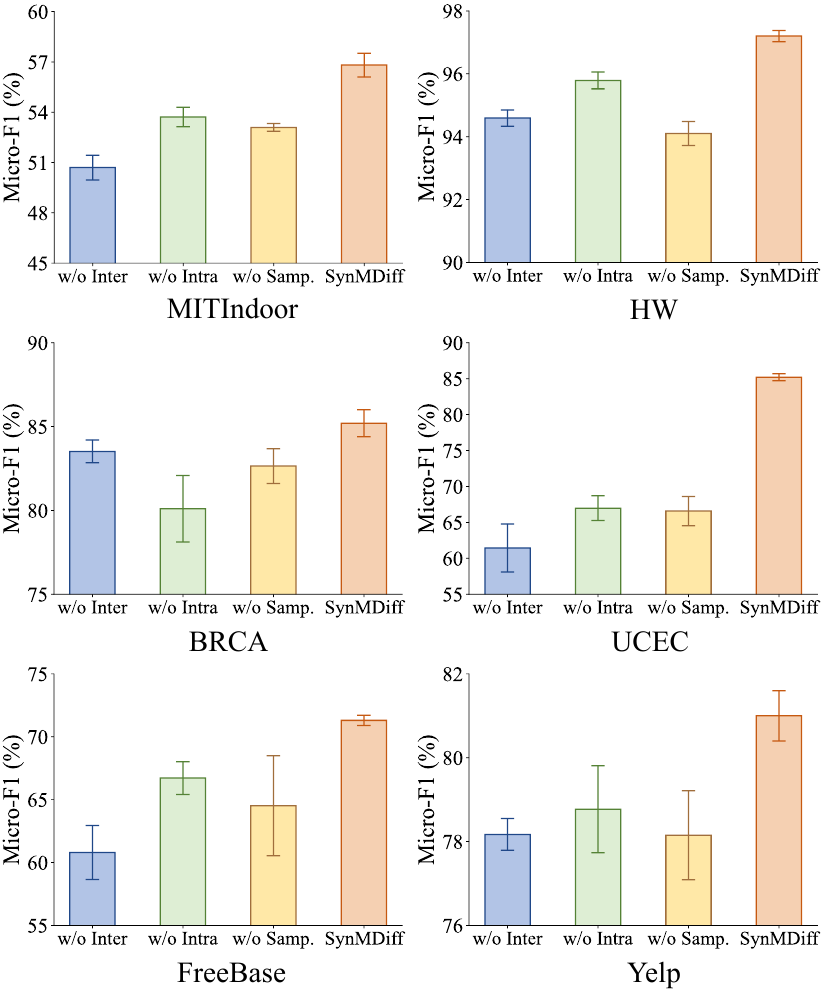}
	\Description{
		Six bar charts arranged in three rows and two columns report
		Micro-F1 scores on MITIndoor, HW, BRCA, UCEC, FreeBase, and Yelp.
		Each chart compares SynMDiff with three ablated variants that remove
		inter-view fusion, intra-view fusion, or the sampling mechanism.
		Error bars indicate variation across experimental runs. The complete
		SynMDiff model achieves the highest Micro-F1 score on every dataset.
		Removing any component reduces performance, with particularly large
		reductions on UCEC and FreeBase, demonstrating that inter-view fusion,
		intra-view fusion, and sampling all contribute to the final results.
	}
	\caption{Ablation study of SynMDiff on six datasets.}
	\label{fig:ablation}
\end{figure}

\section{Discussion and Conclusion}
We propose a synchronous multi-view neural diffusion mechanism to mitigate the conflict of view-specific structural inductive biases that hinder cross-view information propagation.
By modeling pairwise feature interactions in the joint multi-view space as diffusion flows, the proposed method simultaneously captures both intra-view and inter-view adaptive information propagation.
To reduce computational overhead, we design an energy-based topological sampling strategy to prune redundant diffusion flows and introduce a centralized training paradigm based on virtual anchors, enabling efficient inference in large-scale scenarios.

\noindent \textbf{Limitations \& Future Works.}
While we propose a theoretically principled and elegant paradigm, its practical deployment warrants further study, particularly in developing more efficient pruning strategies. Moreover, although the centralized training architecture enables efficient inference and mitigates view-specific biases, it may introduce some information compression. These directions are left for future work.

\begin{acks}
	This work is in part supported by the National Natural Science Foundation of China under Grants U25A20527 and 62276065, and the Fujian Provincial Natural Science Foundation of China under Grant 2024J01510026.
\end{acks}

\balance
\bibliographystyle{ACM-Reference-Format}
\bibliography{references}

\clearpage
\nobalance
\appendix
\numberwithin{equation}{section}
\makeatletter
\begingroup
  \let\saved@shorttitle\shorttitle
  \let\saved@title\@title
  \title{Appendix for ``Synchronous Multi-view Neural Diffusion''}
  \@mktitle
  \let\addresses\@empty
  \num@authorgroups=0\relax
  \@mkauthors
  \twocolumn[\box\mktitle@bx]
  \global\let\shorttitle\saved@shorttitle
  \global\let\@title\saved@title
\endgroup
\@printendtopmatter
\@afterindentfalse
\@afterheading
\makeatother
\pdfbookmark[0]{Appendix}{appendix-title}
\section{Proofs}\label{app:proofs}
\subsection{Proof of Proposition 1}\label{app:proof-prop1}

\begin{restatedpropone}
The asynchronous multi-view diffusion described in Eq.~\eqref{lu diffsuion} constitutes a View-Decoupled Fusion paradigm.
\end{restatedpropone}

\begin{proof}
Let
\begin{equation}
\begin{aligned}
\Phi_i^{(v)}(\mathcal{X}(t))
:=
\frac{\partial \mathbf{x}^{(v)}_i(t)}{\partial t}
&=
\sum_{j=1}^{N} \mathbf{S}^{(v)}_{ij}(t)\big(\mathbf{x}^{(v)}_j(t)-\mathbf{x}^{(v)}_i(t)\big) \\
&+
\sum_{u=1}^{V} \mathbf{P}_{vu}(t)\big(\mathbf{x}^{(u)}_i(t)-\mathbf{x}^{(v)}_i(t)\big)
\end{aligned}
\end{equation}
denote the update operator defined by Eq.~\eqref{lu diffsuion} in the main body. It suffices to verify that $\Phi_i^{(v)}$ satisfies the two properties in Definition~\ref{definition:view_decoupled_fusion}.

We decompose $\Phi_i^{(v)}$ as
\begin{equation}
\Phi_i^{(v)}(\mathcal{X}(t))
=
A_i^{(v)}\big(\mathbf{X}^{(v)}(t)\big)
+
B_i^{(v)}\big(\mathbf{x}^{(1)}_i(t),\dots,\mathbf{x}^{(V)}_i(t)\big),
\end{equation}
where
\begin{equation}
A_i^{(v)}
:=
\sum_{j=1}^{N} \mathbf{S}^{(v)}_{ij}(t)\big(\mathbf{x}^{(v)}_j(t)-\mathbf{x}^{(v)}_i(t)\big),
\end{equation}
and
\begin{equation}
B_i^{(v)}
:=
\sum_{u=1}^{V} \mathbf{P}_{vu}(t)\big(\mathbf{x}^{(u)}_i(t)-\mathbf{x}^{(v)}_i(t)\big).
\end{equation}

\paragraph{(i) View-Independent Evolution.}
By construction, $A_i^{(v)}$ depends only on $\{\mathbf{x}^{(v)}_j(t)\}_{j=1}^{N}$. Therefore, for any $u \neq v$,
\begin{equation}
\frac{\partial A_i^{(v)}}{\partial \mathbf{x}^{(u)}_j(t)} = \mathbf{0}, \quad \forall j,
\end{equation}
which verifies the View-Independent Evolution property in Definition~\ref{definition:view_decoupled_fusion}.

\paragraph{(ii) Intra-instance View Integration.}
The term $B_i^{(v)}$ depends only on the co-indexed features $\{\mathbf{x}^{(u)}_i(t)\}_{u=1}^{V}$. Hence, for any $j \neq i$,
\begin{equation}
\frac{\partial B_i^{(v)}}{\partial \mathbf{x}^{(u)}_j(t)} = \mathbf{0}, \quad \forall u,
\end{equation}
which shows that cross-view interaction is restricted to point-wise integration over co-indexed instances.

Consequently, $\Phi_i^{(v)}$ contains no direct dependence on $\mathbf{x}^{(u)}_j(t)$ with $u \neq v$ and $j \neq i$. Therefore, cross-view and cross-instance dependencies cannot be jointly modeled within a single update step; they can only be induced indirectly through a two-stage mechanism consisting of intra-view propagation followed by point-wise cross-view aggregation.
Thus, the asynchronous multi-view diffusion satisfies both defining properties of View-Decoupled Fusion, and therefore constitutes a View-Decoupled Fusion paradigm.
\end{proof}

\subsection{Proof of Proposition 2}\label{app:proof-prop2}

\begin{restatedproptwo}
The synchronous multi-view diffusion characterized by Eq.~\eqref{global diffusion} constitutes a Global Coupled Fusion.
\end{restatedproptwo}

\begin{proof}
Let
\begin{equation}
\Phi_i^{(v)}(\mathcal{X}(t))
\!:=\!
\frac{\partial \mathbf{x}^{(v)}_i(t)}{\partial t}
\!=\!
\!\sum_{u=1}^{V}\! \alpha^{(u)}(t)\!\sum_{j=1}^{N}\!\mathbf{S}^{(v,u)}_{i,j}(t)\big(\mathbf{x}^{(u)}_j(t)-\mathbf{x}^{(v)}_i(t)\big)
\end{equation}
denote the update operator defined by Eq.~\eqref{global diffusion} in the main body. To establish the proposition, it suffices to verify that $\Phi_i^{(v)}$ satisfies the two properties in Definition~\ref{definition:global_coupled_fusion}.

\paragraph{(i) View-Joint Evolution.}
By Eq.~\eqref{global diffusion} in the main body, the update of $\mathbf{x}^{(v)}_i(t)$ depends explicitly on features from other views through terms of the form
\[
\mathbf{S}^{(v,u)}_{i,j}(t)\big(\mathbf{x}^{(u)}_j(t)-\mathbf{x}^{(v)}_i(t)\big).
\]
Accordingly, the operator $\Phi_i^{(v)}(\mathcal{X}(t))$ is, in general, not independent of $\mathbf{x}^{(u)}_j(t)$ for $u \neq v$; equivalently,
\begin{equation}
\frac{\partial \Phi_i^{(v)}(\mathcal{X}(t))}{\partial \mathbf{x}^{(u)}_j(t)} \neq \mathbf{0}.
\end{equation}
This verifies the View-Joint Evolution property in Definition~\ref{definition:global_coupled_fusion}.

\paragraph{(ii) Cross-instance View Interaction.}
Moreover, the operator $\Phi_i^{(v)}$ aggregates information over all instance-view pairs $(j,u)$. Equivalently, the updated representation of instance $i$ in view $v$ is a function of the global multi-view collection
\begin{equation}
\bigcup_{u=1}^{V}\{\mathbf{x}^{(u)}_1(t),\dots,\mathbf{x}^{(u)}_N(t)\},
\end{equation}
rather than only the co-indexed fiber $\{\mathbf{x}^{(u)}_i(t)\}_{u=1}^{V}$. This verifies the Cross-instance View Interaction property.

Taken together, the synchronous diffusion operator exhibits both inter-view dependency and cross-instance coupling within a single update step. Therefore, the synchronous multi-view diffusion satisfies both defining properties, and hence constitutes a Global Coupled Fusion paradigm.
\end{proof}

\subsection{Proof of Corollary 1}\label{app:proof-cor1}

\begin{restatedcor}
The asynchronous multi-view diffusion in Eq.~\eqref{lu diffsuion} can be regarded as a restricted case of the synchronous multi-view diffusion in Eq.~\eqref{global diffusion}.
\end{restatedcor}

\begin{proof}
Let
\begin{equation}
\Phi_i^{(v)}(\mathcal{X}(t))
:=
\sum_{u=1}^{V} \alpha^{(u)}(t)\sum_{j=1}^{N}\mathbf{S}^{(v,u)}_{i,j}(t)\big(\mathbf{x}^{(u)}_j(t)-\mathbf{x}^{(v)}_i(t)\big)
\end{equation}
denote the synchronous update operator defined by Eq.~\eqref{global diffusion} in the main body. To establish the corollary, it suffices to show that Eq.~\eqref{lu diffsuion} in the main body can be recovered from $\Phi_i^{(v)}$ under a restricted choice of the coupling coefficients.
Consider the following parameterization:
\begin{equation}
\alpha^{(u)}(t) \equiv 1, \qquad \forall u \in \{1,\dots,V\},
\end{equation}
and
\begin{equation}
\mathbf{S}^{(v,u)}_{i,j}(t)
=
\begin{cases}
\mathbf{S}^{(v)}_{ij}(t), & u=v,\\
\mathbf{P}_{vu}(t), & u \neq v,\ j=i,\\
0, & u \neq v,\ j\neq i.
\end{cases}
\end{equation}
Under this restriction, the synchronous operator becomes
\begin{equation}
\begin{aligned}
\Phi_i^{(v)}(\mathcal{X}(t))
&=
\sum_{j=1}^{N}\mathbf{S}^{(v,v)}_{i,j}(t)\big(\mathbf{x}^{(v)}_j(t)-\mathbf{x}^{(v)}_i(t)\big)\\
&+\sum_{u\neq v}\sum_{j=1}^{N}\mathbf{S}^{(v,u)}_{i,j}(t)\big(\mathbf{x}^{(u)}_j(t)-\mathbf{x}^{(v)}_i(t)\big)\\
&=
\sum_{j=1}^{N}\mathbf{S}^{(v)}_{ij}(t)\big(\mathbf{x}^{(v)}_j(t)-\mathbf{x}^{(v)}_i(t)\big)\\
&+\sum_{u\neq v}\mathbf{P}_{vu}(t)\big(\mathbf{x}^{(u)}_i(t)-\mathbf{x}^{(v)}_i(t)\big).
\end{aligned}
\end{equation}
Since the term corresponding to $u=v$ in
\[
\sum_{u=1}^{V}\mathbf{P}_{vu}(t)\big(\mathbf{x}^{(u)}_i(t)-\mathbf{x}^{(v)}_i(t)\big)
\]
vanishes identically, the above expression is equivalently
\begin{equation}
\begin{aligned}
\Phi_i^{(v)}(\mathcal{X}(t))
&=
\sum_{j=1}^{N}\mathbf{S}^{(v)}_{ij}(t)\big(\mathbf{x}^{(v)}_j(t)-\mathbf{x}^{(v)}_i(t)\big)\\
&+
\sum_{u=1}^{V}\mathbf{P}_{vu}(t)\big(\mathbf{x}^{(u)}_i(t)-\mathbf{x}^{(v)}_i(t)\big),
\end{aligned}
\end{equation}
which is exactly Eq.~\eqref{lu diffsuion} in the main body.
Therefore, the asynchronous diffusion is obtained from the synchronous formulation by imposing two structural restrictions: first, cross-view interactions are confined to co-indexed instances ($j=i$); second, all cross-instance interactions across different views ($u\neq v,\ j\neq i$) are removed. Hence, the asynchronous multi-view diffusion is a restricted case of the synchronous multi-view diffusion.
\end{proof}

\subsection{Proof of Theorem 1}\label{app:proof-thm1}
\begin{restatedthm}
(Support Recovery Guarantee)
Let $\mathbf{S}^*$ denote the adjacency matrix induced by the oracle topology $\mathcal{G}^*$, and let $\hat{\mathbf{S}}$ be the truncated adjacency matrix obtained via Eq.~\eqref{truncation}.
Under the incoherence condition,
the probability of structural mismatch between $\hat{\mathbf{S}}$ and $\mathbf{S}^*$ is bounded as:
\begin{equation}
\mathbb{P}\left( \operatorname{supp}(\hat{\mathbf{S}}) \neq \operatorname{supp}(\mathbf{S}^*) \right) \le (NV)^2 \exp\left( - C D \theta^2 \right).
\end{equation}
\end{restatedthm}

\begin{proof}
Let
\[
\mathcal{I}:=\{(v,i): v\in[V],\, i\in[N]\}
\]
denote the set of all view-instance indices, and write
\[
M:=|\mathcal{I}|=NV.
\]
For any \(a,b\in\mathcal{I}\), let \(\widetilde{S}_{ab}\) denote the corresponding entry of the dense affinity matrix before truncation, and let \(S^*_{ab}\) denote the corresponding oracle entry induced by \(\mathcal{G}^*\).

By Eq.~\eqref{truncation} in the main body, the truncated matrix \(\hat{\mathbf{S}}\) is obtained by thresholding \(\widetilde{\mathbf{S}}\) at level \(\theta\). Hence, entrywise,
\begin{equation}
\hat{S}_{ab}\neq 0
\quad\Longleftrightarrow\quad
\widetilde{S}_{ab}\ge \theta.
\end{equation}
Therefore, recovering the support of \(\mathbf{S}^*\) is equivalent to correctly deciding, for every pair \((a,b)\), whether \(\widetilde{S}_{ab}\) lies above or below the threshold \(\theta\).

Define
\[
\mu_{ab}:=\mathbb{E}[\widetilde{S}_{ab}].
\]
Under the incoherence condition, the oracle edges and non-edges are separated around the truncation threshold. Namely, there exists a constant \(c_0>0\) such that for every \(a,b\in\mathcal{I}\),
\begin{equation}\label{eq:margin-separation}
S^*_{ab}\neq 0 \;\Longrightarrow\; \mu_{ab}\ge \theta + c_0\theta,
\qquad
S^*_{ab}=0 \;\Longrightarrow\; \mu_{ab}\le \theta - c_0\theta.
\end{equation}
Moreover, the same incoherence condition implies a concentration inequality for each entry: there exists a constant \(C_1>0\) such that for any \(t>0\),
\begin{equation}\label{eq:entry-concentration}
\mathbb{P}\bigl(|\widetilde{S}_{ab}-\mu_{ab}|\ge t\bigr)
\le 2\exp(-C_1Dt^2).
\end{equation}

Now fix an arbitrary pair \(a,b\in\mathcal{I}\), and define the support recovery error event
\begin{equation}
E_{ab}
:=
\Bigl\{
\mathbf{1}\{\hat{S}_{ab}\neq 0\}
\neq
\mathbf{1}\{S^*_{ab}\neq 0\}
\Bigr\}.
\end{equation}
We bound \(\mathbb{P}(E_{ab})\) by considering two cases.

\paragraph{Case 1: \(S^*_{ab}\neq 0\).}
In this case, \eqref{eq:margin-separation} yields
\[
\mu_{ab}\ge \theta + c_0\theta.
\]
If a support recovery error occurs, then necessarily \(\hat{S}_{ab}=0\), which by the truncation rule means
\[
\widetilde{S}_{ab}<\theta.
\]
Hence,
\[
\mu_{ab}-\widetilde{S}_{ab}
\ge (\theta+c_0\theta)-\theta
= c_0\theta,
\]
and therefore
\[
|\widetilde{S}_{ab}-\mu_{ab}|\ge c_0\theta.
\]
Applying \eqref{eq:entry-concentration} with \(t=c_0\theta\), we obtain
\begin{equation}
\mathbb{P}(E_{ab})
\le
\mathbb{P}\bigl(|\widetilde{S}_{ab}-\mu_{ab}|\ge c_0\theta\bigr)
\le
2\exp(-C_1Dc_0^2\theta^2).
\end{equation}

\paragraph{Case 2: \(S^*_{ab}=0\).}
In this case, \eqref{eq:margin-separation} yields
\[
\mu_{ab}\le \theta - c_0\theta.
\]
If a support recovery error occurs, then necessarily \(\hat{S}_{ab}\neq 0\), namely
\[
\widetilde{S}_{ab}\ge \theta.
\]
Thus,
\[
\widetilde{S}_{ab}-\mu_{ab}
\ge \theta-(\theta-c_0\theta)
= c_0\theta,
\]
which again implies
\[
|\widetilde{S}_{ab}-\mu_{ab}|\ge c_0\theta.
\]
Applying \eqref{eq:entry-concentration} once more gives
\begin{equation}
\mathbb{P}(E_{ab})
\le
2\exp(-C_1Dc_0^2\theta^2).
\end{equation}

Combining the two cases, there exists a constant \(C_2>0\) such that for every \(a,b\in\mathcal{I}\),
\begin{equation}
\mathbb{P}(E_{ab})
\le
\exp(-C_2D\theta^2).
\end{equation}

Finally, the event that the recovered support differs from the oracle support is exactly
\begin{equation}
\Bigl\{
\operatorname{supp}(\hat{\mathbf{S}})
\neq
\operatorname{supp}(\mathbf{S}^*)
\Bigr\}
=
\bigcup_{a,b\in\mathcal{I}} E_{ab}.
\end{equation}
Applying the union bound over all \(M^2=(NV)^2\) entries yields
\begin{equation}
\mathbb{P}\left(
\operatorname{supp}(\hat{\mathbf{S}})
\neq
\operatorname{supp}(\mathbf{S}^*)
\right)
\le
\sum_{a,b\in\mathcal{I}} \mathbb{P}(E_{ab})
\le
M^2 \exp(-C_2D\theta^2).
\end{equation}
Substituting \(M=NV\), we obtain
\begin{equation}
\mathbb{P}\left(
\operatorname{supp}(\hat{\mathbf{S}})
\neq
\operatorname{supp}(\mathbf{S}^*)
\right)
\le
(NV)^2 \exp(-C_2D\theta^2).
\end{equation}
Renaming \(C_2\) as \(C\) completes the proof.
\end{proof}

\section{Computational Complexity Analysis}\label{app:complexity}

In this section, we first analyze the time and space complexity of the asynchronous multi-view diffusion derived from Eq.~\eqref{lu diffsuion} in the main body, and the synchronous multi-view diffusion described in Eq.~\eqref{global diffusion}, respectively. We then further examine the complexity reduction achieved by the proposed topological sampling strategy and the Ego-Net style architecture.
We denote by $N$ the number of instances, by $V$ the number of views, and by $D$ the feature dimension of the latent representation.

\subsection{Complexity of Asynchronous Multi-view Diffusion}

We first consider the asynchronous multi-view diffusion in Eq.~\eqref{lu diffsuion} of the main body. For each target feature $\mathbf{x}_i^{(v)} \in \mathbb{R}^D$, the first term performs intra-view aggregation over all $N$ instances in the $v$-th view. Since each interaction involves a $D$-dimensional vector operation, its cost is $O(ND)$. The second term performs inter-view fusion over all $V$ views for the same instance index $i$, which costs $O(VD)$. Therefore, the cost of updating one feature is
$$
O\big((N+V)D\big).
$$
Since there are $NV$ target features in total, the overall time complexity of one asynchronous diffusion step is
$$
O\big(NV(N+V)D\big)
=
O(N^2VD + NV^2D).
$$

We next analyze the space complexity. Storing the feature tensor $\mathcal{X}=\{\mathbf{X}^{(v)}\}_{v=1}^{V}$ requires $O(NVD)$ memory. The intra-view diffusion coefficients require storing $V$ dense matrices $\mathbf{S}^{(v)} \in \mathbb{R}^{N\times N}$, which costs $O(VN^2)$. In addition, the cross-view coefficients $\mathbf{P}_{vu}$ form a $V \times V$ table, requiring $O(V^2)$ memory. Hence, the total space complexity is
$$
O(NVD + VN^2 + V^2).
$$

In practical scenarios, it is typical that $N \gg D \gg V$. Under this regime, the dominant term in the time complexity is $O(N^2VD)$, and the dominant term in the space complexity is $O(VN^2)$. By further omitting the lower-order dependence on $D$ and $V$, the asynchronous diffusion admits the simplified complexity
$$
\text{Time} = O(N^2), \qquad
\text{Space} = O(N^2).
$$

\subsection{Complexity of Synchronous Multi-view Diffusion}

We now consider the synchronous multi-view diffusion in Eq.~\eqref{global diffusion} of the main body. For each target feature $\mathbf{x}_i^{(v)}$, the update traverses all source pairs $(u,j)$ in the entire multi-view space. Since there are $VN$ such source nodes and each interaction involves a $D$-dimensional vector operation, the cost of updating one target feature is
$$
O(VND).
$$
As there are $NV$ target features in total, the overall time complexity of one synchronous diffusion step becomes
$$
O(NV \cdot VND)
=
O(V^2N^2D).
$$

For the space complexity, storing all node features requires $O(NVD)$ memory. Meanwhile, the dual-coupling coefficients $\mathbf{S}_{i,j}^{(v,u)}$ form a dense tensor over all target-source pairs. The number of such coefficients is
$$
V \times N \times V \times N = V^2N^2,
$$
which gives coefficient storage complexity $O(V^2N^2)$. Therefore, the total space complexity is
$$
O(NVD + V^2N^2).
$$

Again, under the practical regime $N \gg D \gg V$, the dominant terms are $O(V^2N^2D)$ in time and $O(V^2N^2)$ in space. By omitting the lower-order dependence on $D$ and $V$, the synchronous diffusion can be simplified as
$$
\text{Time} = O(N^2), \qquad
\text{Space} = O(N^2).
$$

\subsection{Complexity Reduction via Topological Sampling and Ego-Net Architecture}

A direct implementation of synchronous diffusion is expensive because it requires exhaustive traversal over all $V^2N^2$ target-source pairs. To address this issue, we first apply topological sampling, which truncates the dense topology into a sparse graph. Let $E$ denote the number of retained edges after truncation. Then, instead of aggregating over all possible target-source pairs, diffusion is only performed on the retained edges.
Since each retained edge still involves a $D$-dimensional message passing operation, the total time complexity of one sparse diffusion step becomes
$$
O(ED).
$$

For the space complexity, the sparse topology only needs to store the retained edges and their associated coefficients, which requires $O(E)$ memory. If the feature tensor is also counted, the total memory is
$$
O(NVD + E).
$$
We further introduce the Ego-Net style centralized architecture, where each instance is associated with a virtual anchor $\mathbf{h}_i$. For each anchor $\mathbf{h}_i$, the cost is proportional to the number of retained neighbors, namely
$$
O\left(D\sum_{v=1}^{V}|\mathcal{N}_i^{(v)}|\right).
$$
Summing over all anchors yields the overall time complexity
$$
O\left(
D\sum_{i=1}^{N}\sum_{v=1}^{V}|\mathcal{N}_i^{(v)}|
\right)
=
O(ED).
$$
The corresponding space complexity is still determined by the sparse retained topology, namely $O(E)$, or equivalently $O(NVD+E)$ if the feature storage is included.
Under the practically relevant sparse regime, the average retained degree is bounded by a constant, and thus $E = O(N)$. Substituting this into the above expression gives
$$
\text{Time} = O(ND), \qquad
\text{Space} = O(NVD + E).
$$

Finally, under the practical assumption $N \gg D \gg V$, by omitting the lower-order dependence on $D$ and $V$, the proposed topological sampling and Ego-Net architecture reduce the complexity to
$$
\text{Time} = O(N), \qquad
\text{Space} = O(E).
$$
Therefore, the proposed design reduces the synchronous diffusion from quadratic complexity in the view-joint space to linear-time propagation with sparse topology storage, thereby enabling scalable deployment in large-scale scenarios.

\section{Algorithm}\label{app:algorithm}

This section presents the algorithmic realization of SynMDiff based on the formulation in the main text. The complete procedure is provided in Algorithm~\ref{alg:synmdiff}.

\begin{algorithm}[t!]
\caption{Synchronous Multi-view Neural Diffusion}
\label{alg:synmdiff}
\begin{algorithmic}[1]
\REQUIRE Multi-view dataset $\mathcal{X}=\{\mathbf{X}^{(v)}\in\mathbb{R}^{N\times D^{(v)}}\}_{v=1}^{V}$;
ground-truth labels $\mathbf{Y}$; learning rate $\eta$; similarity threshold $\theta$; diffusion depth $T$
\ENSURE Predicted labels $\hat{\mathbf{Y}}$

\STATE Construct the global diffusion coefficient matrix $\mathbf{S}$ over the joint multi-view space $\mathcal{X}$

\STATE Obtain the sampled topology $\hat{\mathbf{S}}=\mathbf{S}\odot\mathbf{M}(\theta)$ according to Eq.~\eqref{truncation}

\WHILE{not converged}
    \FOR{$v=1$ to $V$}
        \STATE Align the $v$-th view into a shared latent space through a linear layer
        \STATE $\tilde{\mathbf{X}}^{(v)}= f_{\theta_{v}}(\mathbf{X}^{(v)})$
    \ENDFOR

    \FOR{$i=1$ to $N$}
        \STATE Generate the virtual anchor $\mathbf{h}_i^{(0)}$ for the $i$-th instance

        \STATE Construct the Ego-Net topology $\mathcal{T}_i$ from $\hat{\mathbf{S}}$

        \FOR{$t=0$ to $T-1$}
            \STATE Update the virtual anchor by the centralized diffusion rule in Eq.~\eqref{centralized diffusion}
        \ENDFOR

        \STATE Obtain the final representation $\mathbf{z}_i=\mathbf{h}_i^{(T)}$
    \ENDFOR

    \STATE Compute the prediction\\
    $\hat{\mathbf{Y}} = f_{\phi}(\{\mathbf{z}_i\}_{i=1}^{N})$

    \STATE Compute the cross-entropy loss\\
    $\mathcal{L} = - \sum_{i=1}^{N} \sum_{c=1}^{C} Y_{i,c} \log \hat{Y}_{i,c}$

    \STATE Update all trainable parameters by back-propagation with learning rate $\eta$
\ENDWHILE

\STATE \textbf{return} $\hat{\mathbf{Y}}$
\end{algorithmic}
\end{algorithm}

\section{Experimental Details}\label{app:experiments}
\subsection{Dataset Description}\label{app:datasets}
In our experiments, we selected the following five widely-used cancer cohorts:
\begin{itemize}
\setlength{\emergencystretch}{1em}
    \item \textbf{BRCA} supports breast invasive carcinoma PAM50 subtype classification, comprising three omics data types (mRNA, CNV, and Reverse-Phase Protein Array (RPPA)) and containing 511 samples from four subtypes: Luminal A, Luminal B, TNBC, and HER2(+).
    \item \textbf{LGG} supports grade classification in glioma, comprising three omics data types (DNA methylation, miRNA, and mRNA) and containing 524 samples from two subtypes: Grade 2 and Grade 3.
    \item \textbf{UCEC} targets the classification of Uterine Corpus Endometrial Carcinoma, including three omics data types: DNA methylation, miRNA, and mRNA. It comprises 430 samples categorized into EEA, SEA, and MSEAC subtypes.
    \item \textbf{GBMLGG} targets glioma classification, including three types of omics data: DNA methylation, miRNA, and mRNA. It comprises 511 samples grouped into AST, ODG, and OAC subtypes.
    \item \textbf{TCGA} integrates RNA-seq and CNV data, with the CNV processed using the GISTIC2 method (10,845 samples) and RNA-seq normalized for batch effects (11,060 samples). After filtering, the final dataset contains 9,664 samples from 28 subtypes.
\end{itemize}
Table~\ref{modataset} summarizes the statistics of the datasets.
\begin{table}[!htbp]
\centering
\small
\caption{A brief description of cancer subtype datasets.}
\label{modataset}
\resizebox{\columnwidth}{!}{
\begin{tabular}{cccc}
\toprule
Datasets & \# Samples & \#Features & \#Subtypes\\
\midrule
BRCA & 511 & \makecell{mRNA: 1,000 \\ CNV: 1,000 \\ RPPA: 223} & 4 \\
\midrule
LGG & 524 & \makecell{DNA: 2,000 \\ mRNA: 2,000 \\ miRNA: 548} & 2 \\
\midrule
UCEC & 430 & \makecell{DNA: 2,000 \\ mRNA: 2,000 \\ miRNA: 554} & 3 \\
\midrule
GBMLGG & 511 & \makecell{DNA: 2,000 \\ mRNA: 2,000 \\ miRNA: 548} & 3 \\
\midrule
TCGA & 9,664 & \makecell{gene expression: 17,944 \\ CNV: 17,944} & 28 \\
\bottomrule
\end{tabular}
}
\end{table}

\noindent For the heterogeneous graphs, we selected the following five representative datasets:
\begin{itemize}
\setlength{\emergencystretch}{1em}
    \item \textbf{Freebase} constitutes a subset of the broader Freebase knowledge graph, encompassing entities categorized into four distinct types: movies (M), actors (A), directors (D), and writers (W).
    \item \textbf{DBLP} is sourced from the DBLP citation network platform, with each node characterized by 334 attributes. The dataset categorizes nodes into four types: author, paper, term, and conference. In the experiments, the meta-path set \{APA, APCPA, APTPA\} is applied.
    \item \textbf{IMDB} is derived from an online movie database. We extracted a subset containing four types of nodes: movies (M), actors (A), directors (D), and years (Y). We conduct experiments by the meta-path set \{MAM, MDM, MYM\}.
    \item \textbf{YELP} is a subset derived from a merchant review website with four types of nodes, i.e., business, user, service, and level. We generate the meta-path set \{BUB, BLB, BSB\} to conduct experiments.
    \item \textbf{AMiner} represents an academic heterogeneous graph dataset, comprising three types of nodes—paper (P), author (A), and reference (R)—and four types of edges: PA, AP, PR, and RP. The dataset also includes paper categories as labels. The meta-path set \{PAP, PRP\} is used for experimentation.
\end{itemize}
Table~\ref{hgdataset} summarizes the statistics of the datasets.
\begin{table}[htbp]
\centering
\small
\caption{A brief description of heterogeneous graph datasets.}
\label{hgdataset}
\resizebox{\columnwidth}{!}{
\begin{tabular}{cccccc}
\toprule
Datasets & \# Samples & \#Features & \#Views & Meta-Paths & \#Classes\\
\midrule
FreeBase & 43,854 & 3,492 & 3 & \makecell{MAM \\ MDM \\ MWM} & 4\\
\midrule
DBLP & 27,194 & 334 & 3 & \makecell{APA \\ APCPA \\ APTPA} & 4\\
\midrule
IMDB & 12,722 & 1,232 & 3 & \makecell{MAM \\ MDM \\ MYM} & 3\\
\midrule
Yelp &  3,913 & 2,614 & 3 & \makecell{BUB \\ BSB \\ BLB} & 3\\
\midrule
AMiner & 55,783 & 128 & 2 & \makecell{PAP \\ PRP} & 3\\
\bottomrule
\end{tabular}
}
\end{table}

\noindent For the multi-view classification task, we selected the following twelve datasets:
\begin{itemize}
\setlength{\emergencystretch}{1em}
    \item \textbf{HW} is a handwritten digits dataset with six types of features, including Profile-correlation, Fourier-coefficient, Karhunen-Loeve, Morphological, Intensity-averaged, and Zernike Moment, with dimensions 153, 596, 301, 27, 481, and 157, respectively.

    \item \textbf{Youtube} is a video dataset with audio and visual features, including MFCC, Volume Stream, Spectrogram Stream, Cuboids Histogram, Hist Motion Estimate, and Histogram of Oriented Gradients, with dimensions 2000, 64, 1024, 512, 64, and 647, respectively.

    \item \textbf{Scene15} is an image dataset with 15 categories, using LBP, PHOW, and CENTRIST features, with dimensions 1800, 1180, and 1240, respectively.

    \item \textbf{MITIndoor} is a scene dataset with 5,360 images and 67 categories, using PHOW, LBP, CENTRIST, and deep features, with dimensions 4096, 3600, 1770, and 1240, respectively.

    \item \textbf{IAPR} is an image dataset containing 7,855 samples with two views, where each sample is represented by visual and textual features, with dimensions 100 and 100, respectively.

    \item \textbf{Caltech} is an object recognition dataset with 9,144 images and 102 categories. Each sample is described by six types of features, with dimensions 48, 40, 254, 1,984, 512, and 928, respectively.

    \item \textbf{Animals} is an image dataset containing 10,158 samples with two views, where each sample is represented by two types of visual features, both with dimension 4,096.

    \item \textbf{Espgame} is an image dataset with 11,032 samples and two views, where each sample is represented by visual and textual features, with dimensions 100 and 100, respectively.

    \item \textbf{VGGFace} is a face dataset containing 34,027 images with four views, where each sample is described by multiple deep features, with dimensions 944, 576, 512, and 640, respectively.

    \item \textbf{CIFAR-10} is an image dataset with 50,000 samples and 10 categories. Each sample is represented by three types of features, with dimensions 512, 2,048, and 1,024, respectively.

    \item \textbf{NoisyMNIST} is a handwritten digits dataset containing 70,000 samples with two views, where each sample is represented by two noisy feature representations, both with dimension 784.

    \item \textbf{YTF} is a video face dataset containing 286,006 samples with four views, where each sample is described by multiple feature representations, with dimensions 944, 576, 512, and 640, respectively.
\end{itemize}
Table~\ref{mvdataset} summarizes the statistics of the datasets.

\begin{table}[htbp]
\centering
\caption{A brief description of multi-view test datasets.}
\label{mvdataset}
\resizebox{\columnwidth}{!}{
\begin{tabular}{lcccc}
\toprule
Datasets & \# Samples & \# Views & \# Features & \# Classes \\
\midrule
HW    & 2,000  & 6 &  $153/596/301/481/157/27$ & 10 \\
Youtube & 2,000    & 6 &  $2,000/1,024/64/512/64/647$       & 10  \\
Scene15 & 4,485    & 3 &  $1,800/1,180/1,240$ & 15  \\
MITIndoor   & 5,360 & 4 &  $3,600/1,770/1,240/4,096$           & 67 \\
IAPR      & 7,855  & 2 & $100/100$          & 6 \\
Caltech   & 9,144    & 6 &  $48/40/254/1,984/512/928$     & 102  \\
Animals    & 10,158    & 2 & $4,096/4,096$     & 50  \\
Espgame   & 11,032    & 2 & $100/100$    & 7  \\
\midrule
VGGFace   & 34,027    & 4 & $944/576/512/640$     & 50  \\
CIFAR-10   & 50,000    & 3 & $512/2,048/1,024$     & 10  \\
NoisyMNIST  & 70,000    & 2 & $784/784$     & 10  \\
YTF         & 286,006    & 4 & $944/576/512/640$     & 200  \\
\bottomrule
\end{tabular}
}
\end{table}

\subsection{Compared Methods}\label{app:baselines}
For multi-omics cancer subtyping, we selected seven state-of-the-art baselines:
\begin{itemize}
\setlength{\emergencystretch}{1em}
    \item \textbf{SVM} constructs an optimal hyperplane to separate data points with maximum margin and performs classification on the learned representations.
    \item \textbf{Random Forest} constructs an ensemble of decision trees via random sampling and feature selection and performs classification on the learned representations.
    \item \textbf{DeepMO} \cite{lin2020classifying} employs deep neural networks to integrate multi-omics data for accurate breast cancer subtype classification.
    \item \textbf{MOGONET} \cite{wang2021mogonet} jointly learns omics-specific representations and cross-omics correlations via graph convolutional networks for effective multi-omics classification.
    \item \textbf{MoGCN} \cite{li2022mogcn} integrates multi-omics data via graph convolutional networks with autoencoder-based feature extraction and similarity network fusion for accurate cancer subtype classification.
    \item \textbf{Moanna} \cite{lupat2023moanna} integrates multi-omics data via a semi-supervised autoencoder and multi-task learning network for accurate breast cancer subtype prediction.
    \item \textbf{MOSGAT} \cite{wu2024mosgat} integrates multi-omics data via modality specific graph attention networks and cross-modal attention to capture intra- and inter-omics relationships for classification.

\end{itemize}
For heterogeneous graph learning, we selected the following nine advanced compared algorithms:
\begin{itemize}
\setlength{\emergencystretch}{1em}
    \item \textbf{GCN} \cite{kipf2017semi} models a homogeneous graph neural network. We target adjacency matrices generated by different meta-paths fed into GCN module and aggregate the resulting representation through summation.
    \item \textbf{HAN} \cite{wang2019heterogeneous} employed hierarchical attention (node-level: meta-path neighbor importance; and semantic-level: meta-path relevance) to hierarchically aggregate features, enabling interpretable node embeddings in heterogeneous graphs.
    \item \textbf{DMGI} \cite{park2020unsupervised} constructs contrastive learning between the original network and a corrupted network on each meta-path and adds a consensus regularization to fuse node embeddings from different meta-paths.
    \item \textbf{IGNN} \cite{gu2020implicit} employed fixed-point equilibrium equations with implicit state vectors, leveraging Perron-Frobenius theory for well-posedness and implicit differentiation for training, to capture long-range graph dependencies.
    \item \textbf{MRGCN} \cite{huang2020mr} introduced multi-dimensional convolution via Laplacian tensor eigen-decomposition and generalized tensor products, supporting arbitrary unitary transforms for multi-relational graph node classification.
    \item \textbf{SSDCM} \cite{mitra2021semi} is a multiplexed network structure-aware representation learning method designed to maximize the mutual information between local and contextualized global graph summaries.
    \item \textbf{MHGCN} \cite{yu2022multiplex} learned multi-length heterogeneous meta-path interactions through multi-layer convolution aggregation, integrating multi-relational structural signals and attribute semantics into node embeddings for unsupervised and semi-supervised learning.
    \item \textbf{AMOGCN} \cite{chen2024attributed} constructs a fused multi-order adjacency matrix and automatically explores meta-paths involving multi-hop neighbors under the guidance of node semantic information.
    \item \textbf{HMGE} \cite{abdous2023hierarchical} introduced hierarchical aggregation with non-linear dimension combinations and mutual information maximization to unsupervisedly uncover latent structures in high-dimensional multiplex graphs.
\end{itemize}
The following seven state-of-the-art multi-view classification methods are used for comparison with the proposed model:
\begin{itemize}
\setlength{\emergencystretch}{1em}
    \item \textbf{Co-GCN} \cite{li2020co} combined GCNs and co-training to effectively leverage spectral graph information from multiple views through adaptive use of combined Laplacians, enhancing the exploitation of structural data.
    \item \textbf{PDMF} \cite{xu2023progressive} employed a two-stage strategy: pre-training captured multi-view consistency and complementarity via auxiliary representation decoding; fine-tuning leveraged these relations to learn a robust data-to-comprehensive representation mapping.
    \item \textbf{LGCN-FF} \cite{chen2023learnable} consisted of two stages: the first trained a representation from heterogeneous views, and the second enhanced graph fusion using learnable weights and parameterized activation functions.
    \item \textbf{ECMGD} \cite{lu2024towards} integrated energy-constrained graph diffusion to unify multi-view data, enabling inter and intra-view feature flow with an energy function that guided diffusion toward globally consistent representations in GCNs.
    \item \textbf{TUNED} \cite{huang2025trusted} integrates local and global feature neighborhood structures with evidential learning and Markov random fields for robust multi-view classification under uncertainty.
    \item \textbf{KAMSSM} \cite{liao2026static} models knowledge-aware multi-view dynamics via node behavior selection and directional inter-view diffusion to mitigate interference from low-quality views.
    \item \textbf{CoGFormer} \cite{Lai_Li_Wang_Wang_2026} jointly models local and global structural consensus via a cooperative graph transformer with denoising graph convolution and structure-guided attention for robust multi-view learning.
\end{itemize}

\subsection{Implementation and Training Details}\label{app:implementation}

In this section, we provide training details of the proposed method, including model design and hyperparameter settings.

Specifically, we adopt the $k$-nearest neighbors (kNN) algorithm as a concrete instantiation of the truncation procedure in Eq.~\eqref{truncation} of the main text. This choice is motivated by its ability to capture high-quality geometric priors, while remaining computationally efficient and widely used for constructing topological structures from multi-view features \cite{li2020co,chen2023learnable,shi2025information,wang2025MEGNN}.
For heterogeneous graph node classification, we follow the setting of Lu et al.~\cite{lu2024towards}, where a simple GNN (e.g., GCN) is employed to project heterogeneous graph data into a multi-view feature space.
In training, we use the Adam optimizer with cross-entropy loss. The maximum number of epochs is fixed at 200, and the model with the best validation performance is selected for testing. All results are reported as the mean and standard deviation over five runs with different random seeds. The batch size is set to 128 to ensure efficient training on commonly available single-GPU setups.
For hyperparameter selection, we follow the original settings of all baseline methods as specified in their respective papers.
For the proposed SynMDiff, to simplify tuning, we adopt a unified network configuration across all experiments. Specifically, we use a single-layer architecture, where a 512-dimensional MLP is first applied to align multi-view features, followed by a multi-head attention layer with 4 heads of dimension 128, and a feed-forward network (FFN) with hidden dimension 256. The dropout rate is set to 0.2, and the weight decay is fixed at $1\mathrm{e}{-5}$.
We tune the learning rate, the aggregation order (\textit{hops}), and the neighborhood size ($k$) via validation, as these hyperparameters govern the optimization dynamics and the effective receptive field of the induced topology.
The detailed hyperparameter configurations are summarized in Table~\ref{tab:hyperparameters}.

\begin{table}[htbp]
\centering
\small
\caption{Hyperparameter Configurations}
\label{tab:hyperparameters}
\begin{tabular}{p{1.8cm}>{\centering\arraybackslash}p{1.8cm}>{\centering\arraybackslash}p{0.8cm}>{\centering\arraybackslash}p{0.8cm}}
\toprule
Dataset    & learning rate & $hops$ & $k$  \\
\midrule
BRCA       & 0.0005        & 3    & 10 \\
LGG        & 0.0005        & 0    & 5  \\
UCEC       & 0.0005        & 3    & 10 \\
GBMLGG     & 0.0005        & 3    & 10 \\
TCGA       & 0.0005        & 3    & 10 \\
\midrule
FreeBase   & 0.0005        & 3    & 10 \\
DBLP       & 0.00005       & 6    & 10 \\
IMDB       & 0.0005        & 7    & 10 \\
Yelp       & 0.0001        & 5    & 10 \\
Aminer     & 0.0005        & 3    & 10 \\
\midrule
Scene15    & 0.0005        & 1    & 10 \\
YouTube    & 0.0001        & 3    & 10 \\
MITIndoor  & 0.0005        & 2    & 10 \\
HW         & 0.001         & 2    & 10 \\
IAPR       & 0.001         & 3    & 10 \\
Animals    & 0.0005        & 2    & 10 \\
Caltech    & 0.0005        & 1    & 10 \\
ESPGAME    & 0.0005        & 3    & 10 \\
NoisyMNIST & 0.0005        & 4    & 10 \\
YTF        & 0.0001        & 3    & 10 \\
CIFAR-10   & 0.00005       & 1    & 10 \\
VGGFace    & 0.0005        & 1    & 10 \\
\bottomrule
\end{tabular}
\end{table}

\subsection{Additional Ablation Analysis}
\label{sec:app_detailed_ablation}

Section~\ref{sec:ablation} examines the effects of intra-view propagation, inter-view propagation, and topology sampling. We further evaluate three variants on the HW dataset: \textit{Async.}, which replaces synchronous diffusion with its view-decoupled asynchronous counterpart; \textit{w/o $k$-NN}, which removes $k$-NN topology sampling; and \textit{w/o $\mathbf{h}_i$}, which removes the virtual anchor in the centralized architecture. All variants follow the same experimental setting.

\begin{table}[t]
    \centering
    \caption{Component-wise ablation results.}
    \label{tab:app_detailed_ablation}
    \resizebox{\columnwidth}{!}{%
        \begin{tabular}{lcccc}
            \toprule
            Metric
            & Async.
            & w/o $k$-NN
            & w/o $\mathbf{h}_i$
            & SynMDiff \\
            \midrule
            ACC (\%)  & 96.12   & 95.81   & 96.44          & \textbf{97.06} \\
            Time (ms) & 376.25  & 371.15  & 69.60          & \textbf{55.50} \\
            Mem. (MB) & 5376.90 & 5666.89 & \textbf{99.99} & 101.27 \\
            \bottomrule
        \end{tabular}%
    }
\end{table}

As shown in Table~\ref{tab:app_detailed_ablation}, the asynchronous variant decreases ACC by $0.94$ percentage points while requiring approximately $6.78\times$ more time and $53.1\times$ more memory than SynMDiff. Removing $k$-NN sampling leads to the largest performance drop and increases the time and memory costs by approximately $6.69\times$ and $56.0\times$, respectively. This confirms that dense interactions introduce both computational redundancy and semantically irrelevant information.
Removing the virtual anchor $\mathbf{h}_i$ reduces ACC by $0.62$ percentage points and increases the execution time by approximately $25.4\%$, while its memory consumption remains comparable to that of SynMDiff. These results demonstrate that synchronous diffusion, topology sampling, and virtual-anchor-based information relay jointly contribute to the effectiveness and efficiency of the complete model.

\subsection{Sensitivity to the Neighborhood Size}
\label{sec:app_k_sensitivity}

The neighborhood size $k$ controls the trade-off between information coverage and topology sparsity. We vary $k$ from $5$ to $30$ on the HW dataset while keeping all other configurations unchanged.
\begin{table}[t]
    \centering
    \caption{Sensitivity to the neighborhood size $k$.}
    \label{tab:app_k_sensitivity}
    \resizebox{\columnwidth}{!}{%
        \begin{tabular}{lcccccc}
            \toprule
            $k$       & 5              & 10     & 15             & 20     & 25     & 30     \\
            \midrule
            ACC (\%)  & 95.81          & 97.13  & \textbf{97.25} & 96.25  & 96.31  & 95.88  \\
            Time (ms) & \textbf{25.76} & 27.39  & 28.64          & 29.34  & 29.48  & 30.09  \\
            Mem. (MB) & \textbf{99.52} & 100.71 & 99.61          & 100.73 & 99.54  & 99.54  \\
            \bottomrule
        \end{tabular}%
    }
\end{table}
As reported in Table~\ref{tab:app_k_sensitivity}, SynMDiff achieves the highest ACC of $97.25\%$ at $k=15$, while $k=10$ provides comparable performance with slightly lower execution time. A small $k$ may exclude useful complementary information, whereas a large $k$ introduces less relevant neighbors and additional noise.
The execution time increases moderately from $25.76$ ms to $30.09$ ms as $k$ increases, while memory consumption remains stable within $99.52$--$100.73$ MB. An intermediate neighborhood size therefore provides a favorable balance between information coverage, noise suppression, and computational efficiency.

\section{Further Illustration of Synchronous Fusion}
\label{sec:app_sync_motivation}

Neighborhood structures inferred from different views are often heterogeneous and misaligned. Consequently, two semantically related instances may be connected in one view but disconnected in another. In a view-decoupled fusion framework, such missing relations can interrupt intermediate propagation pathways and prevent complementary information from being effectively exchanged.
Figure~\ref{fig:app_sync_motivation} presents a Weisfeiler-Lehman (WL)-inspired toy example with two semantically related instances, $A$ and $B$. The example contains $A_1$, $A_2$, and $B_2$, where the subscript denotes the view, while $B_1$ is omitted to isolate the relevant interaction pattern. The numerals inside the nodes represent abstract feature tokens rather than ground-truth labels. Solid edges denote available interactions, and the crossed dotted edge denotes a missing relation caused by view-specific structural bias.

\begin{figure}[t]
    \centering
    \includegraphics[
        width=\columnwidth,
    ]{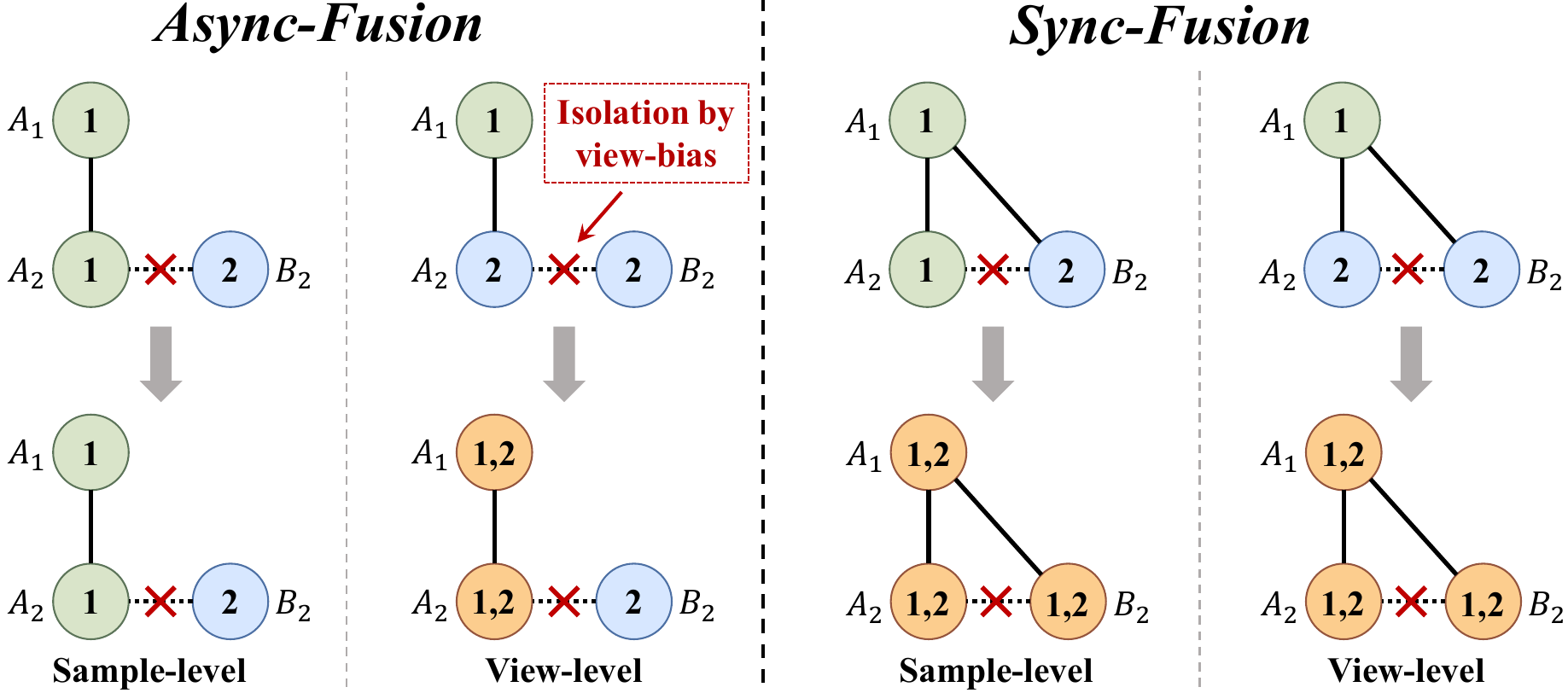}
    \caption{A WL-inspired comparison between asynchronous and synchronous multi-view fusion. Under asynchronous fusion, the missing edge between $A_2$ and $B_2$ interrupts the intermediate propagation pathway and isolates $B_2$ from the complementary information in $A_1$. Synchronous fusion directly models the interaction between $A_1$ and $B_2$, enabling the involved nodes to obtain the joint signature $\{1,2\}$.}
    \Description{A toy graph compares asynchronous and synchronous fusion. A missing edge between A2 and B2 blocks indirect information flow from A1 to B2 in the asynchronous case. A direct cross-view interaction in the synchronous case allows the joint feature signature to reach all involved nodes.}
    \label{fig:app_sync_motivation}
\end{figure}

Under asynchronous fusion, information exchange between $A_1$ and $B_2$ depends on an indirect path through $A_2$. Once the relation between $A_2$ and $B_2$ is absent, the complementary information cannot be propagated at either the sample or view level. The resulting representations are therefore local and incomplete.
In contrast, synchronous fusion operates directly in the joint multi-view space and models the interaction between $A_1$ and $B_2$ without relying on the missing intermediate edge. Both sample- and view-level refinement can consequently recover the joint signature $\{1,2\}$. This WL-inspired example provides an intuitive message-passing interpretation of how synchronous fusion alleviates the restrictions imposed by inconsistent view-specific structures.

\section{Relation to Prior Work and Scope}
\label{sec:app_relation_scope}

SynMDiff is related to prior studies that interpret representation learning as a diffusion process
\cite{chamberlain2021grand,gasteiger2019diffusion,lu2024towards,wudifformer},
where attention mechanisms and graph-based operations are commonly used to estimate diffusion coefficients and implement information propagation. Among these studies, ECMGD~\cite{lu2024towards} is the most closely related multi-view diffusion method. It first performs feature evolution within individual views and then combines the resulting view-specific representations. Despite its diffusion-based interpretation, this sequential organization remains within the View-Decoupled Fusion paradigm, where cross-view interactions depend on intermediate representations shaped by view-specific structural biases. In contrast, SynMDiff formulates the joint multi-view feature space as a unified dynamical system, in which the dual-coupling coefficient $S_{i,j}^{(v,u)}$ directly models the influence of instance $j$ in view $u$ on instance $i$ in view $v$. Intra-view, inter-view, cross-instance, and cross-view interactions are therefore captured within the same diffusion process rather than decomposed into independent stages. To the best of our knowledge, this is the first neural diffusion formulation that models intra- and inter-view interactions in a fully synchronous manner.

Accordingly, the scope of this work primarily concerns the formulation of multi-view fusion rather than the development of a new attention mechanism or graph neural-network backbone. Attention-based diffusivity estimation and graph propagation are adopted as established tools for realizing the generalized diffusion flow, while our focus is the transition from view-decoupled information propagation to globally coupled synchronous fusion. To make joint diffusion computationally practical, we further introduce energy-based topology sampling to remove redundant or semantically irrelevant diffusion pathways and a virtual-anchor architecture to avoid exhaustive cross-view traversal. Together, these designs provide an efficient and scalable realization of the synchronous formulation without altering its central modeling principle.

\end{document}